\documentclass[preprint,12pt,authoryear]{elsarticle}

\usepackage{amssymb}

\usepackage{times}  
\usepackage{helvet}  
\usepackage{courier}  
\usepackage[hyphens]{url}  
\usepackage{graphicx} 
\usepackage{natbib}  
\usepackage{caption}
\usepackage{algorithm}
\usepackage{algorithmic}
\expandafter\let\csname equation*\endcsname\relax
\expandafter\let\csname endequation*\endcsname\relax
\usepackage{amsmath}
\usepackage{amssymb}
\usepackage{booktabs}
\usepackage{multirow}
\usepackage{stfloats}
\usepackage{subfigure}  
\usepackage[colorlinks, linkcolor=black]{hyperref}

\usepackage{amsthm}
\newtheorem{theorem}{Theorem}
\newtheorem{lemma}{Lemma}
\newcommand{\tabincell}[2]{\begin{tabular}{@{}#1@{}}#2\end{tabular}} 
\usepackage[justification=centering]{caption}

\usepackage{adjustbox}
\usepackage{caption}
\usepackage{color, xcolor}
\usepackage{colortbl}  
\usepackage{xcolor}
\usepackage{array} 

\journal{Nerual NetWorks}
\begin{document}

\begin{frontmatter}



\title{Class-Imbalanced Complementary-Label Learning via Weighted Loss}


\author[mymainadress]{Meng Wei}

\author[mymainadress]{Yong Zhou}

\author[mymainadress]{Zhongnian Li}

\author[mymainadress]{Xinzheng Xu\corref{mycorrespondingauthor}}
\cortext[mycorrespondingauthor]{Corresponding author}
\ead{xxzheng@cumt.edu.cn}

\address[mymainadress]{School of Computer Science \& Technology, China University of Mining and Technology, Xuzhou,China} 

\begin{abstract}

Complementary-label learning (CLL) is widely used in weakly supervised classification, but it faces a significant challenge in real-world datasets when confronted with class-imbalanced training samples. In such scenarios, the number of samples in one class is considerably lower than in other classes, which consequently leads to a decline in the accuracy of predictions. Unfortunately, existing CLL approaches have not investigate this problem. To alleviate this challenge, we propose a novel problem setting that enables learning from class-imbalanced complementary labels for multi-class classification. To tackle this problem, we propose a novel CLL approach called Weighted Complementary-Label Learning (WCLL). The proposed method models a weighted empirical risk minimization loss by utilizing the class-imbalanced complementary labels, which is also applicable to multi-class imbalanced training samples. Furthermore, we derive an estimation error bound to provide theoretical assurance. To evaluate our approach, we conduct extensive experiments on several widely-used benchmark datasets and a real-world dataset, and compare our method with existing state-of-the-art methods. The proposed approach shows significant improvement in these datasets, even in the case of multiple class-imbalanced scenarios. Notably, the proposed method not only utilizes complementary labels to train a classifier but also solves the problem of class imbalance.
\end{abstract}


\begin{keyword}
weakly supervised learning \sep complementary labels\sep class imbalanced\sep multi-class classification


\end{keyword}

\end{frontmatter}

\section{Introduction}
Ordinary supervised classification requires precise ground-truth label of each training instance, making it time-consuming for large-scale datasets. To tackle this problem, weakly supervised learning (WSL) has been increasingly explored in the last decades, which allows training a classifier from less costly data. Previous works in WSL include semi-supervised learning \citep{semi-supervised_1, semi-supervised_2, semi-supervised_3, semi-supervised_5}, noisy-label learning \citep{noisy_1, noisy_2, noisy_3, noisy_6, noisy_7}, partial-label learning \citep{pl_1, pl_2, pl_4, pl_5, pl_6}, unlabeled-unlabeled learning \citep{uu_1, uu_2} and positive-unlabeled learning \citep{pu_1, pu_2, pu_3, pu_4, pu_5, pu_6, pu_7}.

In this paper, we focus on complementary-label learning (CLL) \citep{cll_1, cll_2, cll_3, cll_4, cll_5, cll_6, cll_7}, another scenario of WSL, which specifies the class label that the training instance does not belong to. Complementary labels are more accessible than ground-truth labels in some domains \citep{cll_1}. For instance, compared to labeling an instance with a precise ground-truth label from all candidate labels, a complementary label enables crowdsourced workers to choose one label from a set of incorrect classes, thereby simplifying the labeling process. In $ K $-class classification when $ K > $ 2, the complementary label is sampled randomly and uniformly from $ K - 1 $ class. When an instance is labeled with complementary label $i$, it implies that any one of the ordinary labels $ \{1, \ldots i-1, i+1, \ldots, K \} $ could potentially be the ground-truth label.
\begin{figure*}[!htbp]
	\centering
	\begin{adjustbox}{center}
		\includegraphics[width=4.5in]{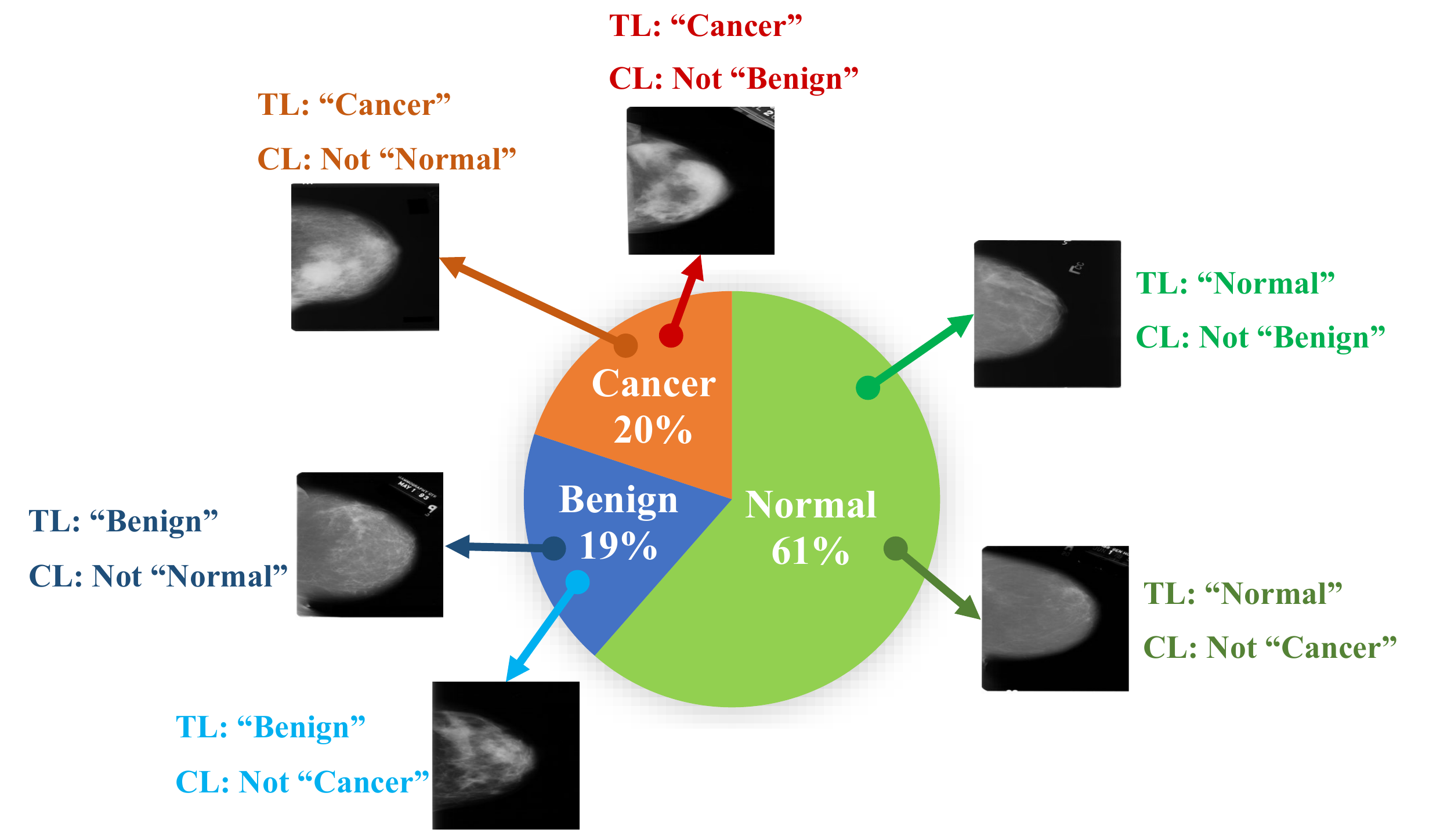}
	\end{adjustbox}
	\caption{The comparative analysis of complementary label and ground-truth label on DDSM datasets. To protect the privacy of patients, collecting complementary labels may be a more viable alternative to obtaining accurate ground-truth labels. In addition, the DDSM dataset exhibits an imbalanced distribution with normal images constituting $61\%$, benign images $19\%$, and cancer images $20\%$. Here, The term "TL" refers to ground-truth label, whereas "CL" denotes complementary laebl.}
	\label{figure_1}
\end{figure*}

However, a significant challenge arises due to class imbalance in practical complementary-label learning. For instance, in the medical disease detection \citep{medical_disease_detection}, there are typically more samples from healthy people than from patients, particularly in privacy-sensitive diseases \citep{application_1}. As illustrated in Figure \ref{figure_1}, the DDSM \footnote{DDSM: \url{http://www.eng.usf.edu/cvprg/Mammography/Database.html}} (Digital Database for Screening Mammography) dataset inherently holds privacy attributes and requires professional knowledge to correctly identify cancer images. Considering the protection of privacy and reducing the difficulty of labeling work, collecting complementary labels may be a more viable alternative to obtaining accurate ground-truth labels. Furthermore, there is an inherent class imbalance in DDSM dataset, wherein a significant portion of the dataset is composed of normal images, accounting for approximately $61\%$. In contrast, only $19\%$ and $20\%$ of the dataset are comprised of benign and cancer images, respectively. This inherent class imbalance presents a challenge for the model to effectively learn from complementary labels. Unfortunately, existing CLL methods have not investigated the issue of learning from class-imbalanced complementary labels, which leads to reduced prediction accuracy, particularly in imbalanced classes. Previous CLL approaches assume a balanced distribution of labeled samples in each class, which is difficult to maintain in real-world datasets\citep{im_1, im_2}. As a result, we propose a novel CLL approach to learn from class-imbalanced complementary labels.

Class imbalance is considered to be one of the main reasons for the performance decline \citep{weighted_loss_6, weighted_loss_1, weighted_loss_2}. Models trained on class-imbalanced datasets tend to favor the majority class and ignore the minority class. To mitigate this problem, supervised classification rebalance data by undersampling or oversampling it, but this results in a huge lack of information or simply duplication of existing data \citep{weighted_loss_2, weighted_loss_3}. Alternatively, cost-sensitive learning assigns higher weights to minority classes to minimize overall losses \citep{weighted_loss_6, weighted_loss_1, weighted_loss_4}. However, these methods are not suitable for CLL since they rely on ground-truth labels. Therefore, this inspires us to propose a novel CLL approach to learn from class-imbalanced complementary labels.

In this study, we investigate the issue of class imbalance in complementary labels for the first time. Specifically, we propose a novel problem setting characterized by a significant disparity in the number of training samples among certain classes. This phenomenon results in an imbalance of complementary labels, which poses a challenge for accurate classification and prediction of existing CLL. To alleviate this issue, we propose a novel CLL approach, called Weighted Complementary-Label Learning (WCLL), which introduces a weighted loss function into the empirical risk to reduce the overall loss. It is theoretically proved that the proposed method can converge to the optimal solution. Additionally, our proposed method can also be applied to multi-class imbalanced setting. To demonstrate the effectiveness of our approach, we conduct comprehensive experiments on benchmark datasets and compared it with state-of-the-art CLL methods. The following is a summary of the major contributions:

\begin{enumerate}
	\item [$ \bullet $] We propose a novel approach for learning from class-imbalanced complementary labels by introducing a weighted loss term to the empirical risk to minimize the total loss value. This approach not only utilizes complementary labels to train a classifier but also effectively alleviate the issue of class imbalance.
	\item [$ \bullet $] It is theoretically proved that the classifier learned from the class-imbalanced complementary labels can converge to $ \mathcal{O}(1/\sqrt{n}) $ with the training samples increasing.
	\item  [$ \bullet $] Our experiments show that better performance can be achieved on class-imbalanced complementary labels. Furthermore, we believe this study can provide useful guidance for researchers interested in bringing imbalance problems into CLL.
\end{enumerate}

The rest of this paper is structured as follows. Section 2 reports on related work. Section 3 gives formal definitions about ordinary multi-class classification and complementary lables learning. Section 4 presents the proposed CLL approach with theoretical analyses and algorithmic specifics. Section 5 reports the results of the comparative experiment. Finally, Section 6 concludes this paper. 

\section{Literature Review}
In this section, we provide an overview of the relevant references, which are mainly categorized into two themes: class imbalance and complementary labels.
\subsection{Class Imbalance} 
Real-world data typically exhibits a class-imbalanced label distribution, which hinders the generalizability of machine learning models \citep{weighted_loss_6,im_1, im_2}. To alleviate this issue, various algorithms have been proposed. The most commonly used strategy is to re-balance the training samples \citep{im_5}. Existing approaches can be roughly divided into three categories.  

The first directly balances the sampling distribution at the data level by over-sampling the minority class or under-sampling the majority class, or both \citep{im_4, im_6, im_7}. \citet{weighted_loss_3} and \cite{weighted_loss_5} utilized universum data to re-balance the class-imbalanced data, but this method is not effective for highly imbalanced data due to the loss of substantial information.

The second alleviates the imbalance problem from the model side. \citet{weighted_loss_2} embedded ensemble learning into deep convolutional neural networks by attaching multiple auxiliary classifiers to different layers of the CNN model. \citet{weighted_loss_4} proposed a large-scale fuzzy least squares twin support vector machines to solve the overfitting problem and avoid the operation of matrix inverse, but this method is only applicable to binary classification.

The third modifies the loss function by assigning higher costs to examples from minority classes \citep{weighted_loss_6, weighted_loss_1}. \citet{weighted_loss_6} utilized an inverse proportional regularization penalty to reweight unbalanced classes. \citet{weighted_loss_1} proposed a hybrid neural network with a cost-sensitive support vector machine to tackle class imbalance, but this method is only suitable for multimodal data. \citet{weighted_loss_7, weighted_loss_10} proposed a rebalancing strategy that utilizes a dynamic weighted loss function, which assigns weights based on class frequency and prediction probability. \citet{weighted_loss_8} proposed a recommended method for deep neural network learning based on item embedding and weighted loss function. \citet{weighted_loss_9} discussed the application of the weighted Kappa loss function in multiclass classification with ordinal data. 

These three distinct approaches provide solutions to alleviate class imbalance from varying perspectives, thereby offering valuable guidance for this issue in CLL. However, these methods rely on ground-truth labels, which is unsuitable for direct application to CLL scenarios.

\subsection{Complementary labels}
Recently, complementary labels learning (CLL) has been widely used in our daily lives \citep{application_1,application_2,application_3}, particularly in data privacy. For certain private questions \citep{cll_1,cll_2}, eliminating false responses can protect the correct one. However, the utilization of complementary labels presents a challenge, as complementary labels are less informative than ordinary labels. Previous research has explored various methods to alleviate this issue, which can be categorized into two branches. 

The first branch aims to model the relationship between the complementary label $\bar{y}$ and the ground-truth label $y$. \citet{cll_1} proposed an unbiased risk estimator (URE) with the assumption that the relationship is unbiased and provided theoretical analysis. However, this approach is only effective for specific loss functions. To alleviate this problem, \citet{cll_3} proposed a general URE framework for arbitrary loss functions. In contrast, \citet{cll_2} used a multi-class Cross-Entropy loss function to solve CLL tasks with the assumption that the relationship is biased, which additionally requires a set of anchor instances for transition probability estimation.
\begin{table*}[!htbp]
	\renewcommand{\arraystretch}{1}
	\label{table_1}
	\caption{Comparison of the proposed methods with previous works}
	\begin{adjustbox}{center}
		\begin{tabular}{c|c|c|c|c}
			\toprule [1pt] 
			Methods & \tabincell{c}{Complementarys \\ label used} &\tabincell{c}{Loss assump. \\ free} & \tabincell{c}{Model assump. \\ free} & Class imbanlance \\
			\midrule
			\citet{cll_1} & $ \checkmark $ & $ \times $ & $ \checkmark $ & $ \times $ \\
			\citet{cll_2} & $ \checkmark $ & $ \times $ & $ \times $  & $ \times $ \\
			\citet{cll_3} & $ \checkmark $ & $ \checkmark $ & $ \checkmark $ & $ \times $ \\
			\citet{cll_7} & $ \checkmark $ & $ \checkmark $ & $ \checkmark $ & $ \times $ \\
			\citet{cll_8} &  $ \checkmark $ & $ \checkmark $ & $ \checkmark $ & $ \times $ \\
			\citet{cll_9} &  $ \checkmark $ & $ \checkmark $ & $ \checkmark $ & $ \times $ \\
			\midrule
			Proposed (WCLL) & $ \checkmark $ & $ \checkmark $ & $ \checkmark $ & $ \checkmark $ \\
			\bottomrule[1pt]
		\end{tabular}
	\end{adjustbox}
\end{table*}

The second branch directly models from the output of classifiers.  \citet{cll_7} proposed a discriminative model that directly models from the output of trained classifiers. The goal is to make the predictive probability of the complementary label approach zero. In addition, there are some interesting setting within the domain of complementary label learning. For instance, \citet{cll_8} explored the introduction of noise in the process of complementary label learning and proposed robust loss functions to learn from noisy complementary labels. \citet{cll_9} conducted a comprehensive study on learning a classifier specifically designed for multiple labeled complementary labeling instances.

Unfortunately, existing CLL approaches do not take into account the problem of class imbalance. Therefore, there is an urgent need to investigate a novel approaches to alleviate this problem. In this regard, we propose a novel CLL approach, called Weighted Complementary-Label Learning, which not only utilizes complementary labels but also effectively alleviate the issue of class imbalance.
\section{Formulation}
In this section, we give notations and review the formulations of ordinary multi-class classification and complementary-label learning.
\subsection{Ordinary Multi-Class Classification}
Let $\mathcal{X} \subset \mathbb{R}^{d}$ denotes the feature space, and $\mathcal{Y} = \{1, 2, \ldots, K\}$ denotes the label space. Consider a random sample $(\textbf{\textit{x}}, y) \sim P(\textbf{\textit{x}}, y)$ drawn from a joint probability distribution over $\mathcal{X}$ and $\mathcal{Y}$. Let $ \mathcal{D}_{L}=\{(\textbf{\textit{x}}_{i}, y_{i}) \}_{i=1}^{N_l} $ denotes the set of training samples, which are sampled independently and identically from $ P(\textbf{\textit{x}},y) $. In ordinary multi-class classification, each instance is associated with a ground-truth label. The goal is to learn a classifier that predicts the probability of the class label, i.e., $P(y|\textbf{\textit{x}})$, and establishes a decision function $f\colon \textbf{\textit{x}} \rightarrow \{1,2,\ldots,K\}$. Here, $f$ denotes an arbitrary function, and $f_{i}(\textbf{\textit{x}})$ denotes the $i$-th element of $f(\textbf{\textit{x}})$. The classification risk for the decision function $f$ is defined with respect to a loss function $\ell$ as:
\begin{equation}
	R(f;\ell)=\mathbb{E}_{(\textbf{\textit{x}}, y)\sim P(\textbf{\textit{x}}, y)}[\ell(f(\textbf{\textit{x}}),y)], 
\end{equation}
where $\mathbb{E}$ refers to the expectation. The approximating empirical risk is defined as:
\begin{equation}
	\hat{R}(f;\ell)=\frac{1}{N_l}\sum_{i=1}^{N_l}\ell(f(\textbf{\textit{x}}_i),y_i), 
\end{equation}
where $ N_l $ refers to the number of training samples.
\subsection{Learning from Complementary Labels}
In contrast to ordinary multi-class classification, complementary labels learning (CLL) assigns a complementary label to each instance, which specifies the label that the instance does not belong to. Let $\mathcal{\bar{Y}} = \{1,2,\ldots,K\}$ denotes the complementary label space. Considering a joint probability distribution $\bar{P}(\textbf{\textit{x}},\bar{y}) \neq P(\textbf{\textit{x}},y)$, where $(\textbf{\textit{x}}, \bar{y})$ denotes a training instance with its corresponding complementary label $\bar{y}$. It is important to note that \citet{cll_1} proposed a URE framework for CLL, assuming an unbiased relationship between $\bar{y}$ and $y$. The classification risk can be defined as follows:
\begin{equation}
	R(f;\ell)=\mathbb{E}_{(\textbf{\textit{x}},\bar{y})\sim\bar P(\textbf{\textit{x}}, \bar{y})}[\bar{\ell}(f(\textbf{\textit{x}}),\bar{y})].
\end{equation}
In addition, \citet{cll_7} modelled directly from the classifier's output by using $\bar{f}(\textbf{\textit{x}})=1 - f(\textbf{\textit{x}})$. The loss function $\bar{\ell}$ can be written as:
\begin{equation}
	\bar{\ell}(f(\textbf{\textit{x}}), \bar{y}) = \ell(1-f(\textbf{\textit{x}}), y).
\end{equation}
Therefore, the risk of CLL classification can be expressed as:
\begin{equation}
	R(f;\ell)=\mathbb{E}_{(\textbf{\textit{x}}, \bar{y}) \sim \bar{P}(\textbf{\textit{x}}, \bar{y})}[\ell(1-f(\textbf{\textit{x}}), y)].
\end{equation}
\section{Method}
This section introduces the problem setting and proposes the weighted model. Additionally, we also derive the estimation error bound of the proposed method to provide a theoretical guarantee.

\subsection{Problem Setting}
\subsubsection{Notation}
Let $ \mathcal{X} \subset \mathbb{R}^{d} $ denotes the feature space, and $ \bar{\mathcal{Y}} = \{1, 2, \ldots, K\} $ denotes the complementary label space. Let $ \bar{\mathcal{D}}_L = \{(\textbf{\textit{x}}_{i}, \bar{y}_{i}) \}_{i=1}^{N_l} $ be sampled from $ \bar{P}(\textbf{\textit{x}},\bar{y}) $ at an imbalanced scale $ \alpha $, where $ N_l $ denotes the number of training samples. The dataset $ \bar{\mathcal{D}}_{L} $ is composed of two subsets: $ \bar{\mathcal{D}}_{min} $ and $ \bar{\mathcal{D}}_{maj} $, where $ \bar{\mathcal{D}}_{min} = \{(\textbf{\textit{x}}, \bar{y}) \mid \bar{y} \in \mathcal{T}_{min} \}$ and $ \bar{\mathcal{D}}_{maj} = \{(\textbf{\textit{x}}, \bar{y}) \mid \bar{y} \in \mathcal{T}_{maj} \}$. Note that $ \bar{\mathcal{D}}_{min} \cup \bar{\mathcal{D}}_{maj} = \bar{\mathcal{D}}_{L},  \bar{\mathcal{D}}_{min} \cap \bar{\mathcal{D}}_{maj} = \varnothing $ and $ N_{min} \ll N_{maj} $. Moreover, let $ \mathcal{T}_{min} = \{\bar{y} | (\textbf{\textit{x}},\bar{y}) \in \bar{\mathcal{D}}_{min} \} $ and the number of each class from $ \mathcal{T}_{min} $ is equal, i.e., $ N_i = N_j, \forall i \neq j, i,j \in \mathcal{T}_{min} $. Similarly, let $ \mathcal{T}_{maj} = \{\bar{y} | (\textbf{\textit{x}},\bar{y}) \in \bar{\mathcal{D}}_{maj} \}  $ and $ N_i = N_j, \forall i \neq j, i,j \in \mathcal{T}_{maj} $. Let $ p $ be the class-imbalanced proportion, which denotes the radio between the number of samples in the classes from $ \mathcal{T}_{maj} $ and $ \mathcal{T}_{min} $, i.e., $ p = \frac{N_i}{N_j}, i \in \mathcal{T}_{maj}, j \in \mathcal{T}_{min} $.

\subsubsection{Class-Imbalanced Setup} 
During the training phase, the dataset $ \bar{\mathcal{D}}_L $ is sampled from $ \bar{P}(\textbf{\textit{x}},\bar{y}) $ at an imbalanced scale denoted by the vector $ \alpha = \{\alpha_1, \alpha_2, \ldots, \alpha_i\} $, where $ \alpha_i $ denotes the number of samples belonging to the $ i $-th class. In this study, the value of $ \alpha_i $ is defined as follows:
\begin{equation}
	\alpha_i = \left\{
	\begin{aligned}
		N_l * \frac{1}{crad(\mathcal{T}_{min}) + p * crad(\mathcal{T}_{maj})}, \quad i \in \mathcal{T}_{min}\\
		N_l * \frac{p}{crad(\mathcal{T}_{min}) + p * crad(\mathcal{T}_{maj})}, \quad i \in \mathcal{T}_{maj}\\
	\end{aligned}
	\right
	.,
\end{equation}
where $ crad(\mathcal{T}_{min}) $ denotes the size of $ \mathcal{T}_{min} $ and $ crad(\mathcal{T}_{maj}) $ denotes the size of $ \mathcal{T}_{maj} $. As mentioned above, a complementary label specifies the class label that the instance does not belong to. In our setting, the number of complementary labels from $ \mathcal{T}_{min} $ is more than other labels from $ \mathcal{T}_{maj} $. We assume that  the number of complementary labels for each class is inversely proportional to the number of samples in that class, i.e., the proportion of complementary label for the $ j $-th class,denoted as $ \pi_j $, can be defined as follows:
\begin{equation}
	\pi_j = \frac{1/{\alpha_j}}{\sum_{i=1}^{K} 1/{\alpha_i}}.
\end{equation}

During the training phase, we utilize class-imbalanced complementarily labeled data. While in the testing phase, we utilize class-balanced data, which also employed by previous CLL approaches. Our goal is to derive an optimal decision function $ f $ to minimize the empirical risk, i.e., $ {f^*} = argmin\hat{R}(f, \ell)$. 
\subsection{The Weighted Model}
Given this novel setting, the data at hand is imbalanced, which increases the difficulty of CLL multi-class classification tasks. As previously mentioned, the classification risk of the proposed CLL setting can be described as:
\begin{equation}
	R(f;\ell)=\mathbb{E}_{(\textbf{\textit{x}},\bar{y}) \sim \bar{P}(\textbf{\textit{x}}, \bar{y})}[\bar{\ell}(f(\textbf{\textit{x}}),\bar{y})],
\end{equation}
where $ \bar{\ell} $ denotes the complementary loss.
\begin{figure*}[!htbp]
	\centering
	\begin{adjustbox}{center}
		\includegraphics[width=5.5in]{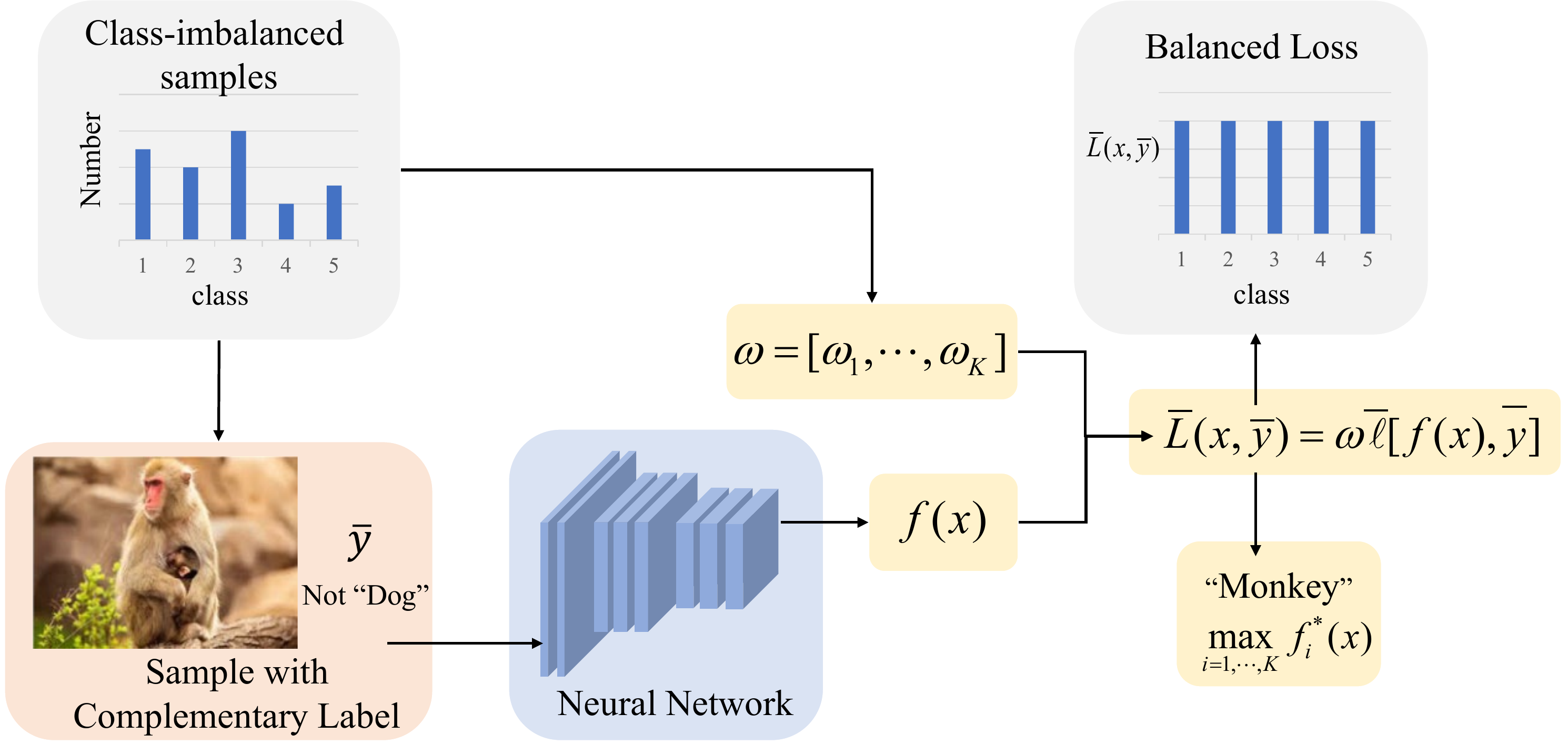}
	\end{adjustbox}
	\caption{The training process of the proposed method. The proposed method involves computing weights for the class-imbalanced dataset and incorporating them into the complementary labeled losses to achieve balanced class losses. During training, our method only use complementary label to train the neural network.}
	\label{figure_2}
\end{figure*}
\begin{algorithm}[htb]
	\caption{ Weighted Complementary-Label Learning}
	\label{alg:Framwork}
	\begin{algorithmic}[1]
		\REQUIRE ~~\\ 
		The set of complementarily labeled training instances, $ \{X_{i}\}, i\in\{1,2,\ldots,K\} $, where $ X_{i} $ denotes the samples labeled as class $ i $, which means the class label that the instance does not belong to;\\
		The number of epochs, $ T $;\\
		An external stochastic optimization algorithm, $ \mathcal{A} $;
		\ENSURE ~~\\ 
		model parameter $ \theta $ for $f(\textbf{\textit{x}}, \theta)$;
		\FOR{$t=1$ to $T$}
		\STATE Shuffle $ \{X_{i}\} $ into $ B $ mini-batches;
		\STATE Calculate $ \omega_j $ accoding to equation (13);
		\FOR{$ b=1 $ to $ B $}
		\STATE Denotes $ r_{i}^{b}(\theta)=-(K-1)\omega_{i}\pi_{i}\ell(f(\textbf{\textit{x}}_b),i)+\sum_{j=1}^{K}\omega_{j}\pi_{j}\ell(f(\textbf{\textit{x}}_b), j) $;
		\STATE Denotes $ L^{b}(\theta)=\sum_{i=1}^{K}r_{i}^{b}(\theta) $
		\STATE Set gradient $ -\bigtriangledown_{\theta}L^{b}(\theta) $
		\STATE Update $ \theta $ by $ \mathcal{A} $
		\ENDFOR
		\ENDFOR
	\end{algorithmic}
\end{algorithm}
The imbalanced sampling distribution creates an imbalance in the generated complementary labels, thereby making it challenging to minimize the total loss due to the increasing class loss with an increasing number of class samples. Additionally, data from minor classes may be characterized as noisy labels,  which makes the classifier training process become more challenging. In order to boost the performance of the trained classifier, we propose an adjustment to the loss function to rebalance the sampling distribution. Specifically, we introduce a weighted loss term for each class sample to minimize the total loss value to complementary loss, as defined below:
\begin{equation}
	\hat{\ell}(f(\textbf{\textit{x}}), \bar{y}) = \omega \bar{\ell} (f(\textbf{\textit{x}}), \bar{y}),
\end{equation}
where $ \omega $ denotes the $ K $-dimensional loss weight vector for $ K $-class, which is inversely proportional to the percentage of complementary labels. In accordance with the proposed problem setting, the complementary loss $ \hat{\ell} $ can be expressed as:

\begin{equation}
	\begin{split}
		\hat{\ell}(f(\textbf{\textit{x}}),j) = -(K-1)\omega_{j}\ell(f(\textbf{\textit{x}}),j) 
		+ \sum_{i=1}^{K}\omega_{i}\ell(f(\textbf{\textit{x}}), i),
	\end{split}
\end{equation}
where $ j\in\{1,\ldots,K\} $, and $ \omega_{j} $  denotes the weight assigned to the $j$-th class. 

By using (10), the classification risk can be described as:
\begin{equation}
	\begin{split}
		R(f;\ell)=\sum_{j=1}^{K}\pi_{j}\mathbb{E}_{(\textbf{\textit{x}}, \bar{y}) \sim \bar{P} (\textbf{\textit{x}}, \bar{y})}[-(K-1)\omega_{j}\ell(f(\textbf{\textit{x}}),j) + \sum_{i=1}^{K}\omega_{i}\ell(f(\textbf{\textit{x}}), i)].
	\end{split}
\end{equation}
where $ \pi_{j}$ denotes the base rate of the $ j $-th class.
Since the training dataset $ \bar{\mathcal{D}}_{L} $ is sampled from $ \bar{P}(\textbf{\textit{x}},\bar{y}) $, the expression (11) can be approximated by 
\begin{equation}
	\begin{split}
		\hat{R}(f;\ell)= \frac{1}{N_l} \{ \sum_{z=1}^{K} \pi_z \sum_{i=i}^{N_z}[-(K-1)\omega_{z}\ell(f(\textbf{\textit{x}}_{i}),z) + \sum_{j=1}^{K}\omega_{j}\ell(f(\textbf{\textit{x}}_{i}), j)] \} .
	\end{split}
\end{equation}
Here, we set $ \omega = [\omega_1, \omega_2, \ldots, \omega_j] $. $ \omega_j $ denotes the $ j $-th element of $ \omega $ and can be defined as follows:
\begin{equation}
	\omega_j = \frac{1/\pi_j}{\sum_{i=1}^{K}1/\pi_i},
\end{equation}
It is important to note that the weights in the proposed model satisfy the condition $\sum_{j=1}^{K}\omega_j = 1$ and $\omega_j \geq 0$. The incorporation of these weighted losses allows the model to effectively learn from training data with limited samples. Consequently, this approach is expected to enhance both the classification accuracy of imbalanced classes and the overall performance of the model. To provide a comprehensive understanding of the proposed method, Algorithm 1 illustrates the overall algorithmic procedure, while Figure \ref{figure_2} illustrates the training process.

\subsection{Estimation Error Bound}
Here, we analyze the generalization estimation error bound for the proposed method. Let $ f $ be the classification vector function in the hypothesis set $ \mathcal{F} $. Let $ \ell[f(\textbf{\textit{x}}), i] \leqslant L_f, \pi_{i} \leqslant \pi_K, \omega_j \leqslant \omega_K, i,j \in \{1, 2, \ldots, K\} $. Using $ L_{\phi} $ to denote the Lipschitz constant \citep{rademacher} of $ \phi $, we can establish the following lemma.
\begin{lemma}
	For any $ \delta > 0 $, with the probability at least $ 1- \delta / 2 $, we have 
	\begin{equation}
		\begin{split}
			\mathop{sup}_{f\in\mathcal{F}} | \hat{R}(f) - R(f) | \leqslant 2 \omega_{K} L_{\phi}  [ \pi_{K} (K-1)  + K] \mathfrak{R}_{N_l}(\mathcal{F}) 
			+ \sqrt{\frac{I^2log\frac{4}{\delta}}{2N_l}},
		\end{split}
	\end{equation}
	where $ R(f) = \mathbb{E}_{\bar{P}} $, $ \hat{R}(f) $ denotes the empirical risk estimator of $ R(f) $, and $ \mathfrak{R}_{N_l}(\mathcal{F}) $ are the Rademacher complexities \citep{rademacher} of $ \mathcal{F} $ for the sampling of size $ N_l $ from $ \bar{P}(\textbf{\textit{x}}, \bar{y}) $, $ I = 2 \omega_KL_f[\pi_K(K-1)+K] $.
\end{lemma}
The proof is given in Appendix. Based on Lemma 1, the estimation error bound can be expressed as follows.
\begin{theorem}
	For any $ \delta > 0 $, with the probability at least $ 1- \delta / 2 $, we have 
	\begin{equation}
		\begin{split}
			R(\hat{f}) - R(f^{\ast})  \leqslant 4 \omega_{K} L_{\phi}  [ \pi_{K} (K-1) + K] \mathfrak{R}_{N_l}(\mathcal{F}) 
			+ \sqrt{\frac{2I^2log\frac{4}{\delta}}{N_l}},
		\end{split}
	\end{equation}
	where $ \hat{f} $ denotes the trained classifier, $ R(f^{\ast}) = \underset{f\in \mathcal{F}}{min}R(f) $.
\end{theorem}
The proof is provided in the Appendix. Lemma 1 and Theorem 1 show that our method exists an estimation error bound. With the deep network hypothesis set $ \mathcal{F} $ fixed, we can establish that $\mathfrak{R}_{N_l}(\mathcal{F}) = \mathcal{O}(1/\sqrt{N_l})$. Consequently, as $ N_l \longrightarrow \infty $, the equation $R(\hat{f}) = R(f^{\ast})$ holds, which prove that our method could converge to the optimal solution.

\begin{figure*}[!htbp]
	\centering
	\begin{adjustbox}{center}
		\subfigure{
			\begin{minipage}[t]{0.4\linewidth}
				\centering
				\includegraphics[width=2.2in]{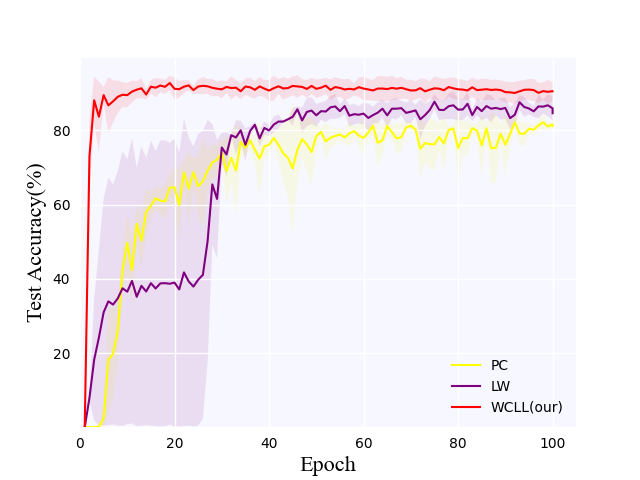}
				\caption*{(a) MNIST, p=2, \\ im-class = 0}
			\end{minipage}
		}%
		\subfigure{
			\begin{minipage}[t]{0.4\linewidth}
				\centering
				\includegraphics[width=2.2in]{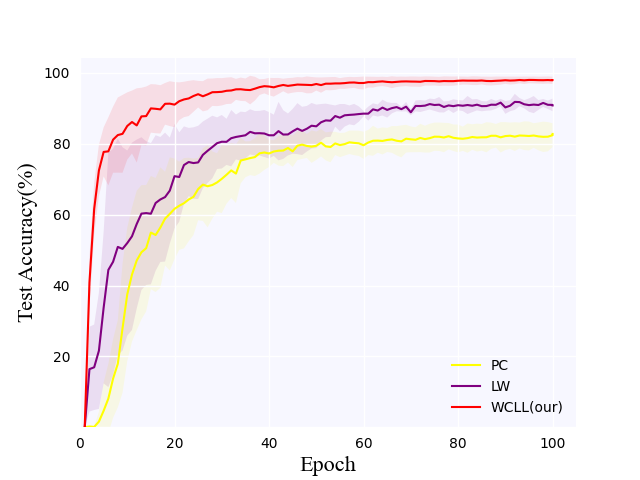}
				\caption*{(b) MNIST, p=2, \\ im-class = 1}
			\end{minipage}
		}%
		\subfigure{
			\begin{minipage}[t]{0.4\linewidth}
				\centering
				\includegraphics[width=2.2in]{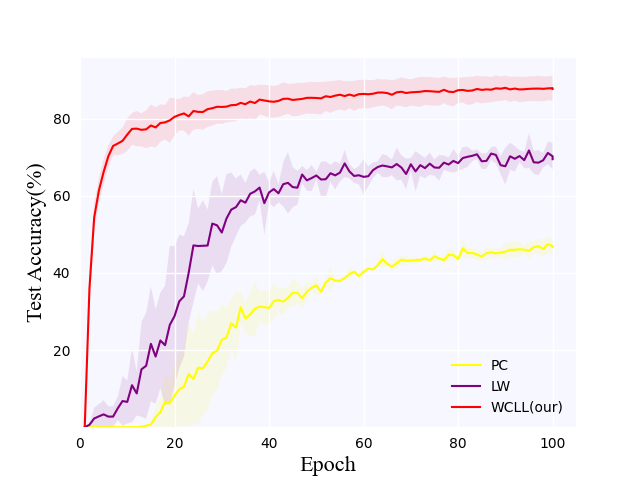}
				\caption*{(c) MNIST, p=2, \\ im-class = 2}
			\end{minipage}
		}%
	\end{adjustbox}

	\begin{adjustbox}{center}
		\subfigure{
			\begin{minipage}[t]{0.4\linewidth}
				\centering
				\includegraphics[width=2.2in]{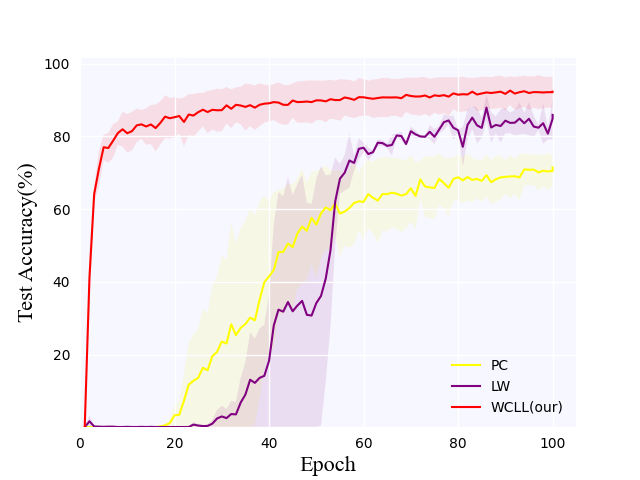}
				\caption*{(d) MNIST, p=2.5,\\im-class = 0}
			\end{minipage}
		}%
		\subfigure{
			\begin{minipage}[t]{0.4\linewidth}
				\centering
				\includegraphics[width=2.2in]{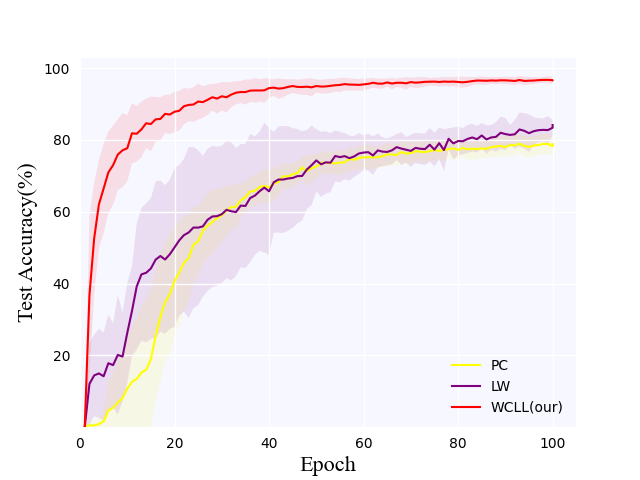}
				\caption*{(e) MNIST, p=2.5,\\im-class = 1}
			\end{minipage}
		}%
		\subfigure{
			\begin{minipage}[t]{0.4\linewidth}
				\centering
				\includegraphics[width=2.2in]{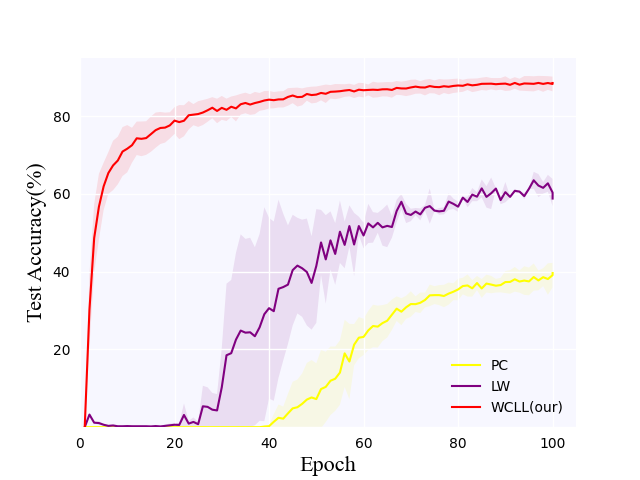}
				\caption*{(f) MNIST, p=2.5,\\im-class = 2}
			\end{minipage}
		}%
	\end{adjustbox}
	
	\begin{adjustbox}{center}
		\subfigure{
			\begin{minipage}[t]{0.4\linewidth}
				\centering
				\includegraphics[width=2.2in]{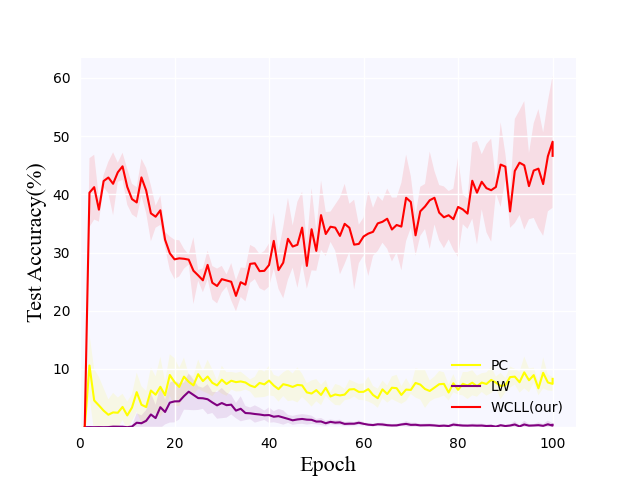}
				\caption*{(g) CIFAR-10, p=5, \\ im-class = 0}
			\end{minipage}
		}%
		\subfigure{
			\begin{minipage}[t]{0.4\linewidth}
				\centering
				\includegraphics[width=2.2in]{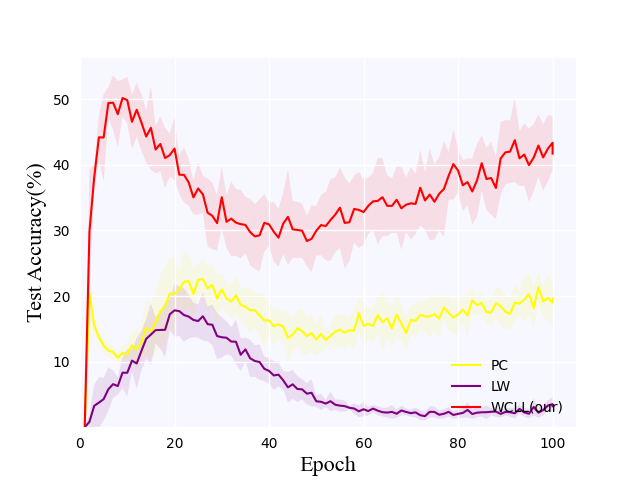}
				\caption*{(h) CIFAR-10, p=5, \\ im-class = 1}
			\end{minipage}
		}%
		\subfigure{
			\begin{minipage}[t]{0.4\linewidth}
				\centering
				\includegraphics[width=2.2in]{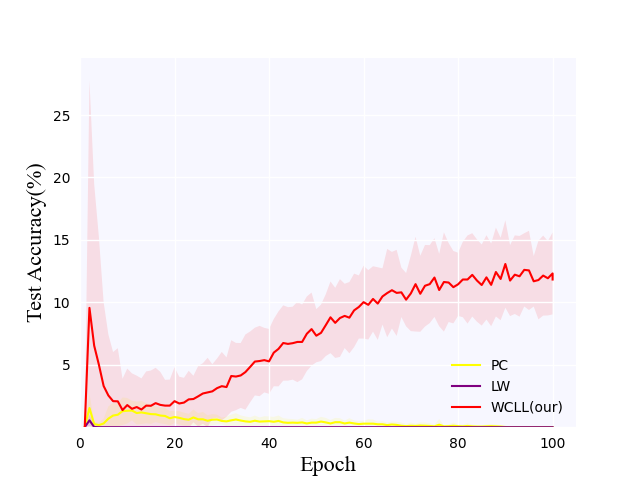}
				\caption*{(i) CIFAR-10, p=5, \\ im-class = 2}
			\end{minipage}
		}%
	\end{adjustbox}
	
	\caption{ Experiments results of test classification accuracy on the MNIST and CIFAR-10 datasets for 5 trials. The dark colors show the mean accuracy of the imbalanced class and the light colors show the standard deviation of the imbalanced class. Here, p denotes the denotes the number of radio between the class from $ \mathcal{T}_{maj} $ and $ \mathcal{T}_{min} $. Im-class refers to the imbalance class. The proposed method has the highest accuracy on imbalanced class.}
	\label{figure_3}
\end{figure*}
\section{Experiments}
This section conducts a comparative study to assess the performance of the proposed method in comparison to state-of-the-art approaches. The proposed method is short for WCLL, which incorporates a class-imbalanced complementary loss function defined in Eq.(12). The mini-batch size and number of epochs were set to 256 and 100, respectively. During the training phase, class-imbalanced complementarily labeled samples generated in the previous section were employed, while class-balanced complementarily labeled data was used for testing phase. All experiments are implemented using PyTorch.

\subsection{Experimental Setting}
\textbf{Datasets: } Consistent with previous studies on CLL \citep{cll_1, cll_3, cll_5, cll_7}, we evaluate our proposed method on widely-used benchmark datasets, including MNIST \citep{MNIST}, CIFAR-10 \citep{CIFAR10}, and Tiny-Imagenet. A summary of the key statistics for each dataset is presented below:
\begin{enumerate}
	\item [$ \bullet $] The MNIST dataset is a handwritten digits dataset, which is composed of 10 classes. Each sample is a $ 28 \times 28 $ grayscale image. The MNIST dataset has 60k training examples and 10k test examples.
	\item [$ \bullet $] The CIFAR-10 dataset has 10 classes of various objects: airplane, automobile, bird, cat, etc. This dataset has 50k training samples and 10k test samples and each sample is a colored image in $ 32 \times 32 \times 3 $ RGB formats.
	\item [$ \bullet $] The Tiny-Imagenet dataset has 200 classes. Each class has 500 training images, 50 validation images, and 50 test images.
\end{enumerate}

\textbf{Compared Methods: } We consider seven current state-of-the-art approaches, which include the following:
\begin{enumerate}
	\item [$ \bullet $] PC\citep{cll_1}: For the first time, \textit{Ishida et al.} proposed the concept of complementary labels and an unbiased estimator for learning from them. However, this method is only effective for specific loss functions.
	\item [$ \bullet $] NN\citep{cll_3}: A non-negative risk estimator to overcome overfitting issue in CLL.
	\item [$ \bullet $] FREE\citep{cll_3}:  An unbiased estimator for learning from complementary labels without limitations of models and loss functions.
	\item [$ \bullet $] LOG\citep{cll_5}: \textit{Feng et al.} proposed a multiple complementarylabel learning method. \textit{Feng et al.} treated multiple complementary labels as a whole, and LOG is an upper-bound surrogate loss function of MAE.
	\item [$ \bullet $] EXP \citep{cll_5}: EXP is also a multiple complementary label learningmethod proposed by \textit{Feng et al.} and EXP is the another upper-bound surrogate loss function of MAE.
	\item [$ \bullet $] LW\citep{cll_7}: \textit{Gao et al.} employed the highly confident predictions in the early stage of learning to boost the performance of succeeding updating of the model. They introduced a weighted loss term to minimize the loss value in CLL.
	\item [$ \bullet $] L-UW\citep{cll_7}: \textit{Gao et al.} proposed the discriminative model that directly estimates $ P(\bar{y} \mid \textbf{\textit{x}}) $ from the output of classifier. They define the prediction probability of complementary label as $ \bar{f}(\textbf{\textit{x}}) = 1 - f(\textbf{\textit{x}}) $.
	\item[$ \bullet $] NCLL\citep{cll_8}: \textit{Ishiguro et al.} analyzed the problem setting where complementary labels may be affected by label noise. They derived conditions for the loss function, ensuring that the learning algorithm remains robust against noise in complementary labels.
	\item[$ \bullet $] ML$\_$CLL\citep{cll_9}: \textit{Gao et al.} proposed an unbiased risk estimator with an estimation error bound to learn a multi-labeled classifier from complementary labeled data.
\end{enumerate}

\begin{table*}[!htbp]
	\renewcommand{\arraystretch}{0.9}
	\caption{Test classification accuracy of total class for 5 trials with mean and standard deviation (mean $ \pm $ std $ \% $) for MNSIT, CIFAR-10 and Tiny-Imagenet dataset. Here, the training data at hand is single class imbalanced. The best performance is shown in bold.}
	\label{table_2}
	\begin{adjustbox}{center}
		\begin{tabular}{c|c|c|c|c|c|c}
			\toprule[1pt] 
			Dataset& P & Class & Method & \tabincell{c} {Imbalanced \\ class = 1} & \tabincell{c}{Imbalanced \\ class = 2}  & \tabincell{c}{Imbalanced \\ class = 3}  \\
			\midrule
			\multirow{14}{*}{MNIST} & \multirow{7}{*}{2} & \multirow{7}{*}{1 $ \sim $ 10}
			& PC \citep{cll_1} & 66.03 $\pm$ 2.98 & 69.73 $\pm$ 3.77 & 67.54 $\pm$ 1.64 \\
			& ~ & ~ & NN \citep{cll_3} & 34.54 $\pm$ 0.80 & 34.15 $\pm$ 1.04 & 35.78 $\pm$ 0.81 \\
			& ~ & ~ & FREE \citep{cll_3} & 59.13 $\pm$ 1.08 & 59.91 $\pm$ 1.41 & 59.88 $\pm$ 0.76 \\
			& ~ & ~ & EXP \citep{cll_5} & 30.93 $\pm$ 5.86 & 30.04 $\pm$ 2.98 & 30.65 $\pm$ 7.32 \\
			& ~ & ~ & LW \citep{cll_7} & 65.48 $\pm$ 4.49 & 70.14 $\pm$ 3.02 & 69.71 $\pm$ 2.05 \\
			& ~ & ~ & L-UW \citep{cll_7} & 66.08 $\pm$ 2.73 & 66.91 $\pm$ 4.25 & 67.04 $\pm$ 6.25 \\
			& ~ & ~ & WCLL (our) & \textbf{72.30 $\pm$ 1.23}  & \textbf{70.55 $\pm$ 1.44} & \textbf{71.07 $\pm$ 1.72} \\
			\cline{2-7}
			& \multirow{7}{*}{2.5} & \multirow{7}{*}{1 $ \sim $ 10}
			& PC \citep{cll_1}  & 65.60 $\pm$ 2.18 & 71.34 $\pm$ 0.83 & 66.68 $\pm$ 1.78 \\
			& ~ & ~ & NN \citep{cll_3}  & 30.50 $\pm$ 1.13 & 30.68 $\pm$ 1.24 & 31.76 $\pm$ 1.39 \\
			& ~ & ~ & FREE \citep{cll_3} & 58.35 $\pm$ 1.64 & 58.53 $\pm$ 1.68 & 59.16 $\pm$ 1.64 \\
			& ~ & ~ & EXP \citep{cll_5} & 30.41 $\pm$ 6.65 & 28.09 $\pm$ 4.80 & 29.45 $\pm$ 3.71 \\
			& ~ & ~ & LW \citep{cll_7}   & 63.10 $\pm$ 3.81 & 64.49 $\pm$ 5.64 & 57.50 $\pm$ 5.20 \\
			& ~ & ~ & L-UW \citep{cll_7} & 66.91 $\pm$ 4.56 & 69.19 $\pm$ 5.26 & 68.50 $\pm$ 3.05 \\
			& ~ & ~ & WCLL (our) & \textbf{71.40 $\pm$ 1.62} & \textbf{72.13 $\pm$ 1.33} & \textbf{73.01 $\pm$ 1.05} \\
			\midrule
			\multirow{10}{*}{CIFAR-10} & \multirow{10}{*}{5} & \multirow{10}{*}{1 $ \sim $ 5}
			& PC \citep{cll_1}  & 54.11 $\pm$ 0.61 & 54.97 $\pm$ 1.01 & 59.08 $\pm$ 0.33 \\
			& ~ & ~ & NN \citep{cll_3}  & 47.67 $\pm$ 0.63 & 44.84 $\pm$ 1.12 & 55.07 $\pm$ 0.88 \\
			& ~ & ~ & FREE \citep{cll_3} & 48.15 $\pm$ 2.70 & 49.43 $\pm$ 1.80 & 54.64 $\pm$ 1.32 \\
			& ~ & ~ & LOG \citep{cll_5} & 57.50 $\pm$ 0.95 & 57.16 $\pm$ 0.54 & 62.74 $\pm$ 0.27 \\
			& ~ & ~ & EXP \citep{cll_5} & 57.70 $\pm$ 0.99 & 57.06 $\pm$ 0.56 & \textbf{62.85 $\pm$ 0.17} \\
			& ~ & ~ & LW \citep{cll_7}  & 53.40 $\pm$ 0.96 & 53.61 $\pm$ 1.15 & 59.60 $\pm$ 0.26 \\
			& ~ & ~ & L-UW \citep{cll_7}  & 56.59 $\pm$ 0.82 & 56.42 $\pm$ 0.89 & 61.42 $\pm$ 0.17 \\
			& ~ & ~ & NCLL \citep{cll_8} & 49.16 $\pm$ 1.67 & 43.32 $\pm$ 0.18 & 57.34 $\pm$ 0.94 \\
			& ~ & ~ & ML$\_$CLL \citep{cll_9} & 55.43 $\pm$ 0.39 & \textbf{59.44 $\pm$ 0.42} & 61.84 $\pm$ 0.10 \\
			& ~ & ~ & WCLL (our) & \textbf{58.02 $\pm$ 0.70} & 56.54 $\pm$ 0.74 & 58.88 $\pm$ 0.41 \\
			\midrule
			\multirow{10}{*}{\tabincell{c}{Tiny-\\Imagenet}} & \multirow{10}{*}{10} & \multirow{10}{*}{1 $ \sim $ 5}
			& PC \citep{cll_1}  & 51.22 $\pm$ 1.02 & 50.66 $\pm$ 2.17 & 55.41 $\pm$ 1.86 \\
			& ~ & ~ & NN \citep{cll_3}  & 43.69 $\pm$ 2.30 & 46.11 $\pm$ 0.87 & 48.42 $\pm$ 1.85 \\
			& ~ & ~ & FREE \citep{cll_3} & 49.40 $\pm$ 2.12 & 48.44 $\pm$ 2.71 & 51.54 $\pm$ 3.88 \\
			& ~ & ~ & LOG \citep{cll_5} & 49.43 $\pm$ 3.08 & 52.84 $\pm$ 1.55 & 59.76 $\pm$ 1.17 \\
			& ~ & ~ & EXP \citep{cll_5} & 48.78 $\pm$ 3.16 & 52.63 $\pm$ 2.27 & \textbf{60.20 $\pm$ 0.91} \\
			& ~ & ~ & LW \citep{cll_7}  & 51.19 $\pm$ 2.47 & 51.84 $\pm$ 1.93 & 56.63 $\pm$ 1.34 \\
			& ~ & ~ & L-UW \citep{cll_7}  & 51.45 $\pm$ 1.27 & 52.86 $\pm$ 1.64 & 58.19 $\pm$ 1.33 \\
			& ~ & ~ & NCLL \citep{cll_8} & 43.29 $\pm$ 0.20 & 43.60 $\pm$ 0.71 & 48.91 $\pm$ 0.80 \\
			& ~ & ~ & ML$\_$CLL \citep{cll_9} & 43.99 $\pm$ 0.30 & 48.40 $\pm$ 0.70 & 48.80 $\pm$ 0.30 \\
			& ~ & ~ & WCLL (our) & \textbf{52.55 $\pm$ 1.33} & \textbf{53.15 $\pm$ 1.63} & 57.31 $\pm$ 1.30 \\
			\bottomrule[1pt]
		\end{tabular}
	\end{adjustbox}
\end{table*}
\begin{table*}[htb]
	\renewcommand{\arraystretch}{0.9}
	\caption{Test classification accuracy (mean $ \pm $ std $ \% $) of total class for 10 trials. Here, the training data at hand is two class imbalanced. The best performance is shown in bold.}
	\label{table_3}
	\begin{adjustbox}{center}
		\begin{tabular}{c|c|c|c|c|c|c}
			\toprule [1pt] 
			Dataset& P & Class & Method & \tabincell{c}{Imbalanced \\ class = 1, 2} & \tabincell{c}{Imbalanced \\ class = 1, 3} & \tabincell{c}{Imbalanced \\ class = 2, 3}  \\
			\midrule
			\multirow{7}{*}{MNIST} & \multirow{7}{*}{2.5} & \multirow{7}{*}{1 $ \sim $ 10}
			& PC \citep{cll_1} & 63.34 $\pm$ 2.16 & 57.51 $\pm$ 1.62 & 64.81 $\pm$ 1.16 \\
			& ~ & ~ & NN \citep{cll_3} & 23.34 $\pm$ 0.77 & 24.58 $\pm$ 1.27 & 23.72 $\pm$ 1.10 \\
			& ~ & ~ & FREE \citep{cll_3} & 48.45 $\pm$ 1.32 & 48.50 $\pm$ 1.49 & 48.33 $\pm$ 1.11 \\
			& ~ & ~ & EXP \citep{cll_5} & 25.76 $\pm$ 3.62 & 29.23 $\pm$ 5.72 & 28.29 $\pm$ 1.45 \\
			& ~ & ~ & LW \citep{cll_7} & 64.32 $\pm$ 2.00 & 62.83 $\pm$ 1.34 & 68.83 $\pm$ 2.48 \\
			& ~ & ~ & L-UW \citep{cll_7} & 64.95 $\pm$ 1.38 & 61.25 $\pm$ 4.06 & 65.34 $\pm$ 3.57 \\
			& ~ & ~ & WCLL (our) & \textbf{68.90 $\pm$ 0.91}  & \textbf{68.76 $\pm$ 1.82} & \textbf{69.78 $\pm$ 1.09} \\
			\midrule
			\multirow{10}{*}{CIFAR-10} & \multirow{10}{*}{5} & \multirow{10}{*}{1 $ \sim $ 5}
			& PC  \citep{cll_1} & 43.53 $\pm$ 0.87 & 49.82 $\pm$ 0.86 & 49.85 $\pm$ 0.93 \\
			& ~ & ~ & NN  \citep{cll_3} & 32.19 $\pm$ 0.43 & 43.73 $\pm$ 0.72 & 41.62 $\pm$ 1.03 \\
			& ~ & ~ & FREE  \citep{cll_3} & 42.02 $\pm$ 2.06 & 45.71 $\pm$ 1.42 & 44.06 $\pm$ 1.12 \\
			& ~ & ~ & LOG  \citep{cll_5} & 47.07 $\pm$ 1.35 & 52.48 $\pm$ 1.05 & 53.60 $\pm$ 0.52 \\
			& ~ & ~ & EXP  \citep{cll_5} & 46.77 $\pm$ 1.97 & 52.16 $\pm$ 1.35 & 52.68 $\pm$ 1.65 \\
			& ~ & ~ & LW  \citep{cll_7} & 42.70 $\pm$ 0.99 & 48.45 $\pm$ 1.05 & 47.81 $\pm$ 1.32 \\
			& ~ & ~ & L-UW  \citep{cll_7} & 43.86 $\pm$ 1.95 & 50.53 $\pm$ 0.99 & 50.18 $\pm$ 1.20 \\
			& ~ & ~ & NCLL \citep{cll_8} & 29.90 $\pm$ 1.70 & 44.66 $\pm$ 0.12 & 41.83 $\pm$ 0.05 \\
			& ~ & ~ & ML$\_$CLL \citep{cll_9} & 43.98 $\pm$ 0.96 & 51.66 $\pm$ 0.10 & 51.94 $\pm$ 0.34 \\
			& ~ & ~ & WCLL (our) & \textbf{50.99 $\pm$ 1.00} & \textbf{54.82 $\pm$ 0.70} & \textbf{53.80 $\pm$ 0.80} \\
			\bottomrule[1pt]
		\end{tabular}
	\end{adjustbox}
\end{table*}

\subsection{Comparison of Single Class Imbalance}
\textbf{Setup: } For MNIST, we conduct experiments to validate the superiority of our method by introducing imbalanced treatment on class 0, class 1, and class 2. Specifically, we considered three scenarios: 1) $ \mathcal{T}_{min}=\{0\}, \mathcal{T}_{maj} = \{1, 2, \ldots, K\} $; 2) $ \mathcal{T}_{min}=\{1\}, \mathcal{T}_{maj} = \{0, 2, \ldots, K\} $; and 3) $ \mathcal{T}_{min}=\{2\}, \mathcal{T}_{maj} = \{0, 1, 3, \ldots, K\} $. A linear model is trained on this dataset with imbalanced proportion $ p $ set to 2 and 2.5. We fix weight decay to $ 1e-4 $ and select the learning rate from $ \{5e-4, 1e-4, 5e-5, 5e-6\} $ for Adam optimizer. 

For CIFAR-10 and Tiny-Imagenet, we set the value of $ K $ to 5 and utilized only the data from the first five classes of the original training set, as well as the test set. We conduct imbalanced treatment on class 0, class 1, and class 2 individually.  Th	e proposed method employed a simple CNN network comprising a 9-layer convolutional network architecture. Each convolutional layer was followed by a Batch Normalization layer and a ReLU layer, with max pooling and drop-out layers applied after every three convolutions. The network concluded with a fully connected layer. The weight decay was set to $ 1e-4 $, and the learning rate was chosen from the set $ \{5e-5, 5e-6\} $.

\textbf{Results:} In Figure \ref{figure_3}, we present the results of the imbalanced class for PC, LW, and the proposed method. As depicted in Figure \ref{figure_3}, our method demonstrates a significantly higher average test classification accuracy compared to the PC and LW methods. Additionally, our method exhibits a smaller standard deviation in test classification accuracy, suggesting its stability compared to other methods. Notably, the proposed method maintains consistent performance when the imbalance proportion changes from 2 to 2.5 on the MNIST dataset, while the accuracy of the NN and PC methods slightly decreases on the imbalanced class.

Table \ref{table_2} presents the classification accuracy for all methods across all classes, based on five trials conducted on various datasets. In view of the suboptimal performance exhibited by NCLL and ML$\_$CLL algorithms when applied to the MNIST dataset, this study exclusively presents their outcomes exclusively on the CIFAR10 and Tiny-imagenet datasets. The proposed method demonstrates consistently high accuracy across all classes in the three datasets, which progressively increase in complexity. However, a slight decrease in accuracy is observed in case 2 and 3 of the CIFAR-10 dataset. We hypothesize that this decrease may be attributed to the presence of a negative term inherent in our method, which leads to overfitting. To alleviate this issue, we plan to conduct further investigations in the future to enhance and refine our model.
\begin{table*}[!htbp]
	\renewcommand{\arraystretch}{0.9}
	\caption{Test classification accuracy (mean $ \pm $ std $ \% $) of total class for 10 trials with more imbalance classes. The best performance is shown in bold.}
	\label{table_4}
	\begin{adjustbox}{center}
		\begin{tabular}{c|c|c|c|c}
			\toprule [1pt] 
			\multirow{4}{*}{Method} & \multicolumn{2}{c|}{MNIST} & \multicolumn{2}{c}{CIFAR-10}\\
			\cline{2-5}
			~ & p=5 & p=5 & p=5 & p=10 \\
			~ & class=1$ \sim $10 & class=1$ \sim $10 & class=1$ \sim $5 & class=1$ \sim $5 \\
			~ & im-class=1,3,5,7,9 & im-class=2,4,6,8,10 & im-class=1,2,3,4 & im-class=1,2,3,4 \\
			\midrule
			PC \citep{cll_1} & 36.80 $\pm$ 1.18 & 36.20 $\pm$ 0.98 & 44.03 $\pm$ 0.58 & 36.27 $\pm$ 0.80 \\
			NN \citep{cll_3} & 14.24 $\pm$ 1.52  & 14.14 $\pm$ 1.16 & 40.89 $\pm$ 2.17 & 33.47 $\pm$ 3.22 \\
			FREE \citep{cll_3} & 20.07 $\pm$ 2.28 & 19.44 $\pm$ 1.38 & 46.50 $\pm$ 1.92 & 40.02 $\pm$ 1.66 \\
			LOG \citep{cll_5} & 9.80 $\pm$ 0.00 & 9.80 $\pm$ 0.00 & 47.58 $\pm$ 1.01 & 39.57 $\pm$ 0.89 \\
			EXP \citep{cll_5} & 16.40 $\pm$ 4.56 & 15.82 $\pm$ 4.49 & 47.77 $\pm$ 1.05 & 40.11 $\pm$ 0.69 \\
			LW \citep{cll_7} & 34.87 $\pm$ 5.64 & 36.20 $\pm$ 5.04 & 47.43 $\pm$ 0.90 & 38.77 $\pm$ 0.88 \\
			L-UW \citep{cll_7} & 31.72 $\pm$ 6.88 & 36.18 $\pm$ 5.64 & 48.16 $\pm$ 0.85 & 39.67 $\pm$ 1.11 \\
			NCLL \citep{cll_8} & -- & -- & 40.83 $\pm$ 0.35 & 34.89 $\pm$ 0.81 \\
			ML$\_$CLL \citep{cll_9} & -- & -- & 41.58 $\pm$ 0.60 & 33.40 $\pm$ 0.62 \\ 
			\midrule
			WCLL (our) & \textbf{55.22 $\pm$ 2.26}  & \textbf{53.61 $\pm$ 2.37} & \textbf{48.34 $\pm$ 0.61} & \textbf{40.38 $\pm$ 0.85} \\
			\bottomrule[1pt]
		\end{tabular}
	\end{adjustbox}
\end{table*}
\subsection{Comparison of Multi-Class Imbalance}
\textbf{Setup:} In this experiment, we assess the performance of the proposed method on multi-class imbalanced data. Specifically, for the MNIST dataset, we conducted imbalanced treatment on class 0 and 1, class 0 and 2, and class 1 and 2, denoted as: 1) $ \mathcal{T}_{min}=\{0, 1\}, \mathcal{T}_{maj} = \{2, \ldots, K\} $; 2) $ \mathcal{T}_{min}=\{0, 2\}, \mathcal{T}_{maj} = \{1, 3, \ldots, K\} $; 3) $ \mathcal{T}_{min}=\{1, 2\}, \mathcal{T}_{maj} = \{0, 3, \ldots, K\} $. In each setting, we conducted 10 trials with an imbalanced proportion of 2.5. We employed a linear model with weight-decay set to $1e-4$ and learning rate selected from $\{5e-5,5e-6\}$. 

For CIFAR-10, we utilize only the first five classes of the training set and the test set. Imbalanced treatment was applied to class 0 and 1, class 0 and 2, and class 1 and 2, with imbalanced proportion set to 5 and 10. The same model utilized in the previous section was employed in this experiment.

\textbf{Results:} Table \ref{table_3} reports the test classification accuracy for 10 trials on two multi-class imbalanced datasets. Since compared methods exhibit poor performance on imbalanced classes, we focus our experiment on the test classification accuracy across all classes. The proposed method demonstrates a significant improvement in accuracy when multiple classes are imbalanced compared to other methods. 
\subsection{Additional Experiments}
\begin{table*}[htbp]
	\renewcommand{\arraystretch}{1}
	\caption{Test classification accuracy (mean $ \pm $ std $ \% $) of total class for 10 trials on MNIST datatset. The best performance is shown in bold.}
	\label{table_5}
	\begin{adjustbox}{center}
		\begin{tabular}{c|c|c|c|c|c|c|c}
			\toprule [1pt] 
			Random& Bias=0.05 & Bias=0.1 & Bias=0.15 & Bias=0.2 & Bias=0.3 & Bias=0.5 & WCLL (our) \\
			\midrule
			\tabincell{c}{14.86 \\ $\pm$ 3.55} & \tabincell{c}{54.66 \\ $\pm$ 0.92} & \tabincell{c}{54.94 \\ $\pm$ 1.05} & \tabincell{c}{54.45 \\ $\pm$ 1.12} & \tabincell{c}{54.04 \\ $\pm$ 1.09} & \tabincell{c}{53.47 \\ $\pm$ 1.17} & \tabincell{c}{52.78 \\ $\pm$ 1.20} & \textbf{\tabincell{c}{72.30 \\ $\pm$ 1.23}} \\
			\bottomrule[1pt]
		\end{tabular}
	\end{adjustbox}
\end{table*}
\begin{table*}[htbp]
	\renewcommand{\arraystretch}{1}
	\caption{Test classification accuracy (mean $ \pm $ std $ \% $) of total class for 10 trials on MNIST datatset. The best performance is shown in bold.}
	\label{table_6}
	\begin{adjustbox}{center}
		\begin{tabular}{c|c|c|c}
			\toprule [1pt] 
			P & Under-sampling & Over-sampling & WCLL (our) \\
			\midrule
			2 & 49.56 $\pm$ 1.75 & 46.08 $\pm$ 2.56 & \textbf{72.30 $\pm$ 1.23} \\
			\midrule
			2.5 & 48.66 $\pm$ 1.26 & 46.02 $\pm$ 2.42 & \textbf{71.40 $\pm$ 1.62} \\
			\midrule
			5 & 44.38 $\pm$ 2.78 & 45.87 $\pm$ 2.57 & \textbf{68.71 $\pm$ 1.61} \\
			\bottomrule[1pt]
		\end{tabular}
	\end{adjustbox}
\end{table*}

\textbf{Extreme Case:} To further demonstrate the effectiveness of our method on imbalanced data, we test our method's performance on data with more imbalanced classes. For MNIST, we tested two cases: 1) $ K = 10, T_{min} = \{1, 3, 5, 7, 9\} $, and 2) $ K = 10, T_{min} = \{2, 4, 6, 8, 10\} $. For CIFAR-10, we tested two cases: 1) $ K = 5, p = 5, T_{min} = \{1, 2, 3, 4\} $, and 2) $ K = 5, p = 10, T_{min} = \{1, 2, 3, 4\} $. The results are shown in Table \ref{table_4}, where we report the test classification accuracy of the total class. As demonstrated by the results, our method maintains high accuracy even in more extreme cases of imbalanced data, which further proves the superiority of our model.

\textbf{Issue of Assumption:} To verify our assumption on class imbalance of complementary labels, we conduct an additional experiment where we compare the results when $ \pi_j $ is not satisfied. We develop two distinct weighting strategies to alleviate this scenario. The first approach involves randomly generating a set of weights to the loss function. The second approach involves adding a bias to the proposed method, specifically, $ \pi_j' = \pi_j + bias, $ where the bias value was set to 0.05, 0.1, 0.15, 0.2, 0.3, and 0.5. The results of the test accuracy for 10 trials using the two different approaches are presented in Table \ref{table_5}. The performance is noticeably lower when $ \pi_j $ was not met, confirming the significance of our assumption on class imbalance of complementary labels.

\begin{table*}[!htbp]
	\renewcommand{\arraystretch}{0.9}
	\caption{Test classification accuracy of DDSM dataset. The best performance is shown in bold.}
	\label{table_7}
	\begin{adjustbox}{center}
		\begin{tabular}{c|c|c|c}
			\toprule [1pt] 
			{Method} & total accuracy & benign accuracy & cancer accuracy \\
			\midrule
			PC \citep{cll_1} & 34.62 & 4.23 & 25.38 \\
			NN \citep{cll_3} & 36.54 & 0.39 & 12.31 \\
			FREE \citep{cll_3} & 45.00 & 33.46 & 3.46 \\
			LOG \citep{cll_5} & 40.26 & 3.85 & 25.77 \\
			EXP \citep{cll_5} & 38.85 & -- & 28.08 \\
			LW \citep{cll_7} & 38.85 & -- & 28.85 \\
			L-UW \citep{cll_7} & 39.87 & 0.38 & 33.46 \\
			NCLL \citep{cll_8} & 36.41 & -- & 10.77 \\
			ML$\_$CLL \citep{cll_9} & 37.32 & -- & 15.77 \\ 
			\midrule
			WCLL (our) & \textbf{57.56} & \textbf{76.54} & \textbf{34.23} \\
			\bottomrule[1pt]
		\end{tabular}
	\end{adjustbox}
\end{table*}
\textbf{Issue of Samping:} We conducted additional experiments to compare the effectiveness of reweighting and resampling techniques on the same datasets. For undersampling, we removed the excess majority class samples and balanced them with the number of minority class samples. Conversely, oversampling was implemented by increasing the samples of the minority class to match the majority class. Table \ref{table_6} presents the comparative results of these two techniques on the MNIST dataset. As depicted in the table \ref{table_6}, the proposed method outperforms the under-sampling and over-sampling techniques by a significant margin.

\textbf{Real-world Scenario:} The performance results on the DDSM dataset are presented in Table \ref{table_7}, showcasing the outcome on a real-world dataset. The proposed method exhibits superior performance on the DDSM dataset, surpassing all baseline methods across all classes with a remarkable accuracy improvement of $17.56\%$.  Notably, our method also demonstrates superior performance on the two unbalanced classes, namely benign and cancer accuracy. Specifically, our proposed method achieves a $43.08\%$ increase in benign accuracy and a $0.77\%$ increase in cancer accuracy compared to the best baseline method. This is a significant improvement, as all baseline methods failed to improve accuracy on both classes. Our findings clearly highlight the superiority of our proposed method.

\section{Conclusion}
Class imbalance is a common problem in real-world datasets, yet existing CLL approaches have not sufficiently alleviated this issue. Therefore, we propose a novel CLL problem setting specifically for class-imbalanced data. Our proposed approach is a weighted CLL method that learns from class-imbalanced data and has been theoretically proven to converge to the optimal solution. Our method not only leverages complementary labels to train a classifier, but also effectively alleviates the issue of class imbalance. The experimental results demonstrate that our method also outperforms other state-of-the-art methods when dealing with multi-class imbalanced samples.

It is noteworthy that the proposed method suffers from overfitting issue due to the negative problem. In the future, we consider designing a regularization penalty to alleviate this issue and improve the model's performance. Furtherore, the complementary labels were chosen at random, which also appears in other weakly supervised learning approaches. Taking the initiative to select some complementary labels that can enhance the model's performance would be interesting. In addition, dynamically weighted loss has been shown to be effective in dealing with class imbalance \citep{weighted_loss_7,weighted_loss_10}, and it would be valuable to investigate how to create a dynamically weighted strategy for data with class-imbalanced complementary labels.

\section*{Acknowledgement}
This work is supported by the National Natural Science Foundation of China (No.61976217), the Fundamental Research Funds for the Central Universities (No.2019XKQYMS87), the Science and Technology Planning Project of Xuzhou (No.KC21193)




\bibliographystyle{elsarticle-harv} 
\bibliography{infer}

\begin{thebibliography}{55}
\expandafter\ifx\csname natexlab\endcsname\relax\def\natexlab#1{#1}\fi
\providecommand{\url}[1]{\texttt{#1}}
\providecommand{\href}[2]{#2}
\providecommand{\path}[1]{#1}
\providecommand{\DOIprefix}{doi:}
\providecommand{\ArXivprefix}{arXiv:}
\providecommand{\URLprefix}{URL: }
\providecommand{\Pubmedprefix}{pmid:}
\providecommand{\doi}[1]{\href{http://dx.doi.org/#1}{\path{#1}}}
\providecommand{\Pubmed}[1]{\href{pmid:#1}{\path{#1}}}
\providecommand{\bibinfo}[2]{#2}
\ifx\xfnm\relax \def\xfnm[#1]{\unskip,\space#1}\fi
\bibitem[{Buda et~al.(2018)Buda, Maki and Mazurowski}]{im_5}
\bibinfo{author}{Buda, M.}, \bibinfo{author}{Maki, A.},
  \bibinfo{author}{Mazurowski, M.A.}, \bibinfo{year}{2018}.
\newblock \bibinfo{title}{A systematic study of the class imbalance problem in
  convolutional neural networks}.
\newblock \bibinfo{journal}{Neural networks} \bibinfo{volume}{106},
  \bibinfo{pages}{249--259}.
\bibitem[{Byrd and Lipton(2019)}]{im_6}
\bibinfo{author}{Byrd, J.}, \bibinfo{author}{Lipton, Z.}, \bibinfo{year}{2019}.
\newblock \bibinfo{title}{What is the effect of importance weighting in deep
  learning?}, in: \bibinfo{booktitle}{International Conference on Machine
  Learning}, pp. \bibinfo{pages}{872--881}.
\bibitem[{Chapel et~al.(2020)Chapel, Alaya and Gasso}]{pu_4}
\bibinfo{author}{Chapel, L.}, \bibinfo{author}{Alaya, M.Z.},
  \bibinfo{author}{Gasso, G.}, \bibinfo{year}{2020}.
\newblock \bibinfo{title}{Partial optimal tranport with applications on
  positive-unlabeled learning}.
\newblock \bibinfo{journal}{Advances in Neural Information Processing Systems}
  \bibinfo{volume}{33}, \bibinfo{pages}{2903--2913}.
\bibitem[{Chapelle et~al.(2006)Chapelle, Sch{\"o}lkopf and
  Zien}]{semi-supervised_1}
\bibinfo{author}{Chapelle, O.}, \bibinfo{author}{Sch{\"o}lkopf, B.},
  \bibinfo{author}{Zien, A.}, \bibinfo{year}{2006}.
\newblock \bibinfo{title}{A discussion of semi-supervised learning and
  transduction}, in: \bibinfo{booktitle}{Semi-supervised learning}, pp.
  \bibinfo{pages}{473--478}.
\bibitem[{Chen et~al.(2021)Chen, Duan, Kang and Qiu}]{weighted_loss_2}
\bibinfo{author}{Chen, Z.}, \bibinfo{author}{Duan, J.}, \bibinfo{author}{Kang,
  L.}, \bibinfo{author}{Qiu, G.}, \bibinfo{year}{2021}.
\newblock \bibinfo{title}{Class-imbalanced deep learning via a class-balanced
  ensemble}.
\newblock \bibinfo{journal}{IEEE transactions on neural networks and learning
  systems} \bibinfo{volume}{33}, \bibinfo{pages}{5626--5640}.
\bibitem[{Chou et~al.(2020)Chou, Niu, Lin and Sugiyama}]{cll_4}
\bibinfo{author}{Chou, Y.T.}, \bibinfo{author}{Niu, G.}, \bibinfo{author}{Lin,
  H.T.}, \bibinfo{author}{Sugiyama, M.}, \bibinfo{year}{2020}.
\newblock \bibinfo{title}{Unbiased risk estimators can mislead: A case study of
  learning with complementary labels}, in: \bibinfo{booktitle}{International
  Conference on Machine Learning}, pp. \bibinfo{pages}{1929--1938}.
\bibitem[{Dong et~al.(2018)Dong, Gong and Zhu}]{im_1}
\bibinfo{author}{Dong, Q.}, \bibinfo{author}{Gong, S.}, \bibinfo{author}{Zhu,
  X.}, \bibinfo{year}{2018}.
\newblock \bibinfo{title}{Imbalanced deep learning by minority class
  incremental rectification}.
\newblock \bibinfo{journal}{IEEE transactions on pattern analysis and machine
  intelligence} \bibinfo{volume}{41}, \bibinfo{pages}{1367--1381}.
\bibitem[{Du~Plessis et~al.(2015)Du~Plessis, Niu and Sugiyama}]{pu_2}
\bibinfo{author}{Du~Plessis, M.}, \bibinfo{author}{Niu, G.},
  \bibinfo{author}{Sugiyama, M.}, \bibinfo{year}{2015}.
\newblock \bibinfo{title}{Convex formulation for learning from positive and
  unlabeled data}, in: \bibinfo{booktitle}{International conference on machine
  learning}, pp. \bibinfo{pages}{1386--1394}.
\bibitem[{Du~Plessis et~al.(2013)Du~Plessis, Niu and Sugiyama}]{uu_1}
\bibinfo{author}{Du~Plessis, M.C.}, \bibinfo{author}{Niu, G.},
  \bibinfo{author}{Sugiyama, M.}, \bibinfo{year}{2013}.
\newblock \bibinfo{title}{Clustering unclustered data: Unsupervised binary
  labeling of two datasets having different class balances}, in:
  \bibinfo{booktitle}{2013 Conference on Technologies and Applications of
  Artificial Intelligence}, pp. \bibinfo{pages}{1--6}.
\bibitem[{Du~Plessis et~al.(2014)Du~Plessis, Niu and Sugiyama}]{pu_1}
\bibinfo{author}{Du~Plessis, M.C.}, \bibinfo{author}{Niu, G.},
  \bibinfo{author}{Sugiyama, M.}, \bibinfo{year}{2014}.
\newblock \bibinfo{title}{Analysis of learning from positive and unlabeled
  data}.
\newblock \bibinfo{journal}{Advances in neural information processing systems}
  \bibinfo{volume}{27}, \bibinfo{pages}{703--711}.
\bibitem[{Feng et~al.(2020)Feng, Kaneko, Han, Niu, An and Sugiyama}]{cll_5}
\bibinfo{author}{Feng, L.}, \bibinfo{author}{Kaneko, T.}, \bibinfo{author}{Han,
  B.}, \bibinfo{author}{Niu, G.}, \bibinfo{author}{An, B.},
  \bibinfo{author}{Sugiyama, M.}, \bibinfo{year}{2020}.
\newblock \bibinfo{title}{Learning with multiple complementary labels}, in:
  \bibinfo{booktitle}{International Conference on Machine Learning}, pp.
  \bibinfo{pages}{3072--3081}.
\bibitem[{Fernando and Tsokos(2021)}]{weighted_loss_7}
\bibinfo{author}{Fernando, K.R.M.}, \bibinfo{author}{Tsokos, C.P.},
  \bibinfo{year}{2021}.
\newblock \bibinfo{title}{Dynamically weighted balanced loss: class imbalanced
  learning and confidence calibration of deep neural networks}.
\newblock \bibinfo{journal}{IEEE Transactions on Neural Networks and Learning
  Systems} \bibinfo{volume}{33}, \bibinfo{pages}{2940--2951}.
\bibitem[{Ganaie et~al.(2022a)Ganaie, Tanveer, Initiative
  et~al.}]{weighted_loss_5}
\bibinfo{author}{Ganaie, M.}, \bibinfo{author}{Tanveer, M.},
  \bibinfo{author}{Initiative, A.D.N.}, et~al., \bibinfo{year}{2022}a.
\newblock \bibinfo{title}{Knn weighted reduced universum twin svm for class
  imbalance learning}.
\newblock \bibinfo{journal}{Knowledge-Based Systems} \bibinfo{volume}{245},
  \bibinfo{pages}{108578}.
\bibitem[{Ganaie et~al.(2022b)Ganaie, Tanveer and Lin}]{weighted_loss_4}
\bibinfo{author}{Ganaie, M.}, \bibinfo{author}{Tanveer, M.},
  \bibinfo{author}{Lin, C.T.}, \bibinfo{year}{2022}b.
\newblock \bibinfo{title}{Large-scale fuzzy least squares twin svms for class
  imbalance learning}.
\newblock \bibinfo{journal}{IEEE Transactions on Fuzzy Systems}
  \bibinfo{volume}{30}, \bibinfo{pages}{4815--4827}.
\bibitem[{Gao et~al.(2023)Gao, Xu and Zhang}]{cll_9}
\bibinfo{author}{Gao, Y.}, \bibinfo{author}{Xu, M.}, \bibinfo{author}{Zhang,
  M.L.}, \bibinfo{year}{2023}.
\newblock \bibinfo{title}{Learning from noisy labels with complementary loss
  functions}, in: \bibinfo{booktitle}{Proceedings of the 32nd International
  Joint Conference on Artificial Intelligence}.
\bibitem[{Gao and Zhang(2021)}]{cll_7}
\bibinfo{author}{Gao, Y.}, \bibinfo{author}{Zhang, M.L.}, \bibinfo{year}{2021}.
\newblock \bibinfo{title}{Discriminative complementary-label learning with
  weighted loss}, in: \bibinfo{booktitle}{International Conference on Machine
  Learning}, pp. \bibinfo{pages}{3587--3597}.
\bibitem[{Gerych et~al.(2022)Gerych, Hartvigsen, Buquicchio, Agu and
  Rundensteiner}]{pu_7}
\bibinfo{author}{Gerych, W.}, \bibinfo{author}{Hartvigsen, T.},
  \bibinfo{author}{Buquicchio, L.}, \bibinfo{author}{Agu, E.},
  \bibinfo{author}{Rundensteiner, E.}, \bibinfo{year}{2022}.
\newblock \bibinfo{title}{Recovering the propensity score from biased positive
  unlabeled data} , \bibinfo{pages}{6694--6702}.
\bibitem[{Ghosh et~al.(2017)Ghosh, Kumar and Sastry}]{noisy_2}
\bibinfo{author}{Ghosh, A.}, \bibinfo{author}{Kumar, H.},
  \bibinfo{author}{Sastry, P.S.}, \bibinfo{year}{2017}.
\newblock \bibinfo{title}{Robust loss functions under label noise for deep
  neural networks}, in: \bibinfo{booktitle}{Proceedings of the AAAI conference
  on artificial intelligence}, pp. \bibinfo{pages}{1919--1925}.
\bibitem[{Golovnev et~al.(2019)Golovnev, P{\'a}l and Szorenyi}]{uu_2}
\bibinfo{author}{Golovnev, A.}, \bibinfo{author}{P{\'a}l, D.},
  \bibinfo{author}{Szorenyi, B.}, \bibinfo{year}{2019}.
\newblock \bibinfo{title}{The information-theoretic value of unlabeled data in
  semi-supervised learning}, in: \bibinfo{booktitle}{International Conference
  on Machine Learning}, pp. \bibinfo{pages}{2328--2336}.
\bibitem[{Gong et~al.(2022)Gong, Yuan and Bao}]{pl_6}
\bibinfo{author}{Gong, X.}, \bibinfo{author}{Yuan, D.}, \bibinfo{author}{Bao,
  W.}, \bibinfo{year}{2022}.
\newblock \bibinfo{title}{Partial multi-label learning via large margin nearest
  neighbour embeddings} , \bibinfo{pages}{6729--6736}.
\bibitem[{Guo and Li(2022)}]{im_7}
\bibinfo{author}{Guo, L.Z.}, \bibinfo{author}{Li, Y.F.}, \bibinfo{year}{2022}.
\newblock \bibinfo{title}{Class-imbalanced semi-supervised learning with
  adaptive thresholding}, in: \bibinfo{booktitle}{International Conference on
  Machine Learning}, pp. \bibinfo{pages}{8082--8094}.
\bibitem[{Han et~al.(2020)Han, Niu, Yu, Yao, Xu, Tsang and Sugiyama}]{noisy_7}
\bibinfo{author}{Han, B.}, \bibinfo{author}{Niu, G.}, \bibinfo{author}{Yu, X.},
  \bibinfo{author}{Yao, Q.}, \bibinfo{author}{Xu, M.}, \bibinfo{author}{Tsang,
  I.}, \bibinfo{author}{Sugiyama, M.}, \bibinfo{year}{2020}.
\newblock \bibinfo{title}{Sigua: Forgetting may make learning with noisy labels
  more robust}, in: \bibinfo{booktitle}{International Conference on Machine
  Learning}, pp. \bibinfo{pages}{4006--4016}.
\bibitem[{He and Garcia(2009)}]{im_4}
\bibinfo{author}{He, H.}, \bibinfo{author}{Garcia, E.A.}, \bibinfo{year}{2009}.
\newblock \bibinfo{title}{Learning from imbalanced data}.
\newblock \bibinfo{journal}{IEEE Transactions on knowledge and data
  engineering} \bibinfo{volume}{21}, \bibinfo{pages}{1263--1284}.
\bibitem[{Hu et~al.(2021)Hu, Le, Liu, Ji, Ma, Zhao and Yan}]{pu_5}
\bibinfo{author}{Hu, W.}, \bibinfo{author}{Le, R.}, \bibinfo{author}{Liu, B.},
  \bibinfo{author}{Ji, F.}, \bibinfo{author}{Ma, J.}, \bibinfo{author}{Zhao,
  D.}, \bibinfo{author}{Yan, R.}, \bibinfo{year}{2021}.
\newblock \bibinfo{title}{Predictive adversarial learning from positive and
  unlabeled data}, in: \bibinfo{booktitle}{Proceedings of the AAAI Conference
  on Artificial Intelligence}, pp. \bibinfo{pages}{7806--7814}.
\bibitem[{Ishida et~al.(2017)Ishida, Niu, Hu and Sugiyama}]{cll_1}
\bibinfo{author}{Ishida, T.}, \bibinfo{author}{Niu, G.}, \bibinfo{author}{Hu,
  W.}, \bibinfo{author}{Sugiyama, M.}, \bibinfo{year}{2017}.
\newblock \bibinfo{title}{Learning from complementary labels}.
\newblock \bibinfo{journal}{Advances in neural information processing systems}
  \bibinfo{volume}{30}, \bibinfo{pages}{5639--5649}.
\bibitem[{Ishida et~al.(2019)Ishida, Niu, Menon and Sugiyama}]{cll_3}
\bibinfo{author}{Ishida, T.}, \bibinfo{author}{Niu, G.},
  \bibinfo{author}{Menon, A.}, \bibinfo{author}{Sugiyama, M.},
  \bibinfo{year}{2019}.
\newblock \bibinfo{title}{Complementary-label learning for arbitrary losses and
  models}, in: \bibinfo{booktitle}{International Conference on Machine
  Learning}, pp. \bibinfo{pages}{2971--2980}.
\bibitem[{Ishiguro et~al.(2022)Ishiguro, Ishida and Sugiyama}]{cll_8}
\bibinfo{author}{Ishiguro, H.}, \bibinfo{author}{Ishida, T.},
  \bibinfo{author}{Sugiyama, M.}, \bibinfo{year}{2022}.
\newblock \bibinfo{title}{Learning from noisy complementary labels with robust
  loss functions}.
\newblock \bibinfo{journal}{IEICE TRANSACTIONS on Information and Systems}
  \bibinfo{volume}{105}, \bibinfo{pages}{364--376}.
\bibitem[{Izmailov et~al.(2020)Izmailov, Kirichenko, Finzi and
  Wilson}]{semi-supervised_5}
\bibinfo{author}{Izmailov, P.}, \bibinfo{author}{Kirichenko, P.},
  \bibinfo{author}{Finzi, M.}, \bibinfo{author}{Wilson, A.G.},
  \bibinfo{year}{2020}.
\newblock \bibinfo{title}{Semi-supervised learning with normalizing flows}, in:
  \bibinfo{booktitle}{International Conference on Machine Learning}, pp.
  \bibinfo{pages}{4615--4630}.
\bibitem[{Johnson and Khoshgoftaar(2019)}]{medical_disease_detection}
\bibinfo{author}{Johnson, J.M.}, \bibinfo{author}{Khoshgoftaar, T.M.},
  \bibinfo{year}{2019}.
\newblock \bibinfo{title}{Survey on deep learning with class imbalance}.
\newblock \bibinfo{journal}{Journal of Big Data} \bibinfo{volume}{6},
  \bibinfo{pages}{1--54}.
\bibitem[{Kaneko et~al.(2019)Kaneko, Sato and Sugiyama}]{application_1}
\bibinfo{author}{Kaneko, T.}, \bibinfo{author}{Sato, I.},
  \bibinfo{author}{Sugiyama, M.}, \bibinfo{year}{2019}.
\newblock \bibinfo{title}{Online multiclass classification based on prediction
  margin for partial feedback}.
\newblock \bibinfo{journal}{arXiv preprint arXiv:1902.01056} .
\bibitem[{Kim and Sohn(2020)}]{weighted_loss_1}
\bibinfo{author}{Kim, K.H.}, \bibinfo{author}{Sohn, S.Y.},
  \bibinfo{year}{2020}.
\newblock \bibinfo{title}{Hybrid neural network with cost-sensitive support
  vector machine for class-imbalanced multimodal data}.
\newblock \bibinfo{journal}{Neural Networks} \bibinfo{volume}{130},
  \bibinfo{pages}{176--184}.
\bibitem[{de~La~Torre et~al.(2018)de~La~Torre, Puig and
  Valls}]{weighted_loss_9}
\bibinfo{author}{de~La~Torre, J.}, \bibinfo{author}{Puig, D.},
  \bibinfo{author}{Valls, A.}, \bibinfo{year}{2018}.
\newblock \bibinfo{title}{Weighted kappa loss function for multi-class
  classification of ordinal data in deep learning}.
\newblock \bibinfo{journal}{Pattern Recognition Letters} \bibinfo{volume}{105},
  \bibinfo{pages}{144--154}.
\bibitem[{LeCun et~al.(1998)LeCun, Bottou, Bengio and Haffner}]{MNIST}
\bibinfo{author}{LeCun, Y.}, \bibinfo{author}{Bottou, L.},
  \bibinfo{author}{Bengio, Y.}, \bibinfo{author}{Haffner, P.},
  \bibinfo{year}{1998}.
\newblock \bibinfo{title}{Gradient-based learning applied to document
  recognition}.
\newblock \bibinfo{journal}{Proceedings of the IEEE} \bibinfo{volume}{86},
  \bibinfo{pages}{2278--2324}.
\bibitem[{Liu et~al.(2019)Liu, Miao, Zhan, Wang, Gong and Yu}]{im_2}
\bibinfo{author}{Liu, Z.}, \bibinfo{author}{Miao, Z.}, \bibinfo{author}{Zhan,
  X.}, \bibinfo{author}{Wang, J.}, \bibinfo{author}{Gong, B.},
  \bibinfo{author}{Yu, S.X.}, \bibinfo{year}{2019}.
\newblock \bibinfo{title}{Large-scale long-tailed recognition in an open
  world}, in: \bibinfo{booktitle}{Proceedings of the IEEE/CVF Conference on
  Computer Vision and Pattern Recognition}, pp. \bibinfo{pages}{2537--2546}.
\bibitem[{Lv et~al.(2020)Lv, Xu, Feng, Niu, Geng and Sugiyama}]{pl_4}
\bibinfo{author}{Lv, J.}, \bibinfo{author}{Xu, M.}, \bibinfo{author}{Feng, L.},
  \bibinfo{author}{Niu, G.}, \bibinfo{author}{Geng, X.},
  \bibinfo{author}{Sugiyama, M.}, \bibinfo{year}{2020}.
\newblock \bibinfo{title}{Progressive identification of true labels for
  partial-label learning}, in: \bibinfo{booktitle}{International Conference on
  Machine Learning}, pp. \bibinfo{pages}{6500--6510}.
\bibitem[{Ma et~al.(2018)Ma, Wang, Houle, Zhou, Erfani, Xia, Wijewickrema and
  Bailey}]{noisy_3}
\bibinfo{author}{Ma, X.}, \bibinfo{author}{Wang, Y.}, \bibinfo{author}{Houle,
  M.E.}, \bibinfo{author}{Zhou, S.}, \bibinfo{author}{Erfani, S.},
  \bibinfo{author}{Xia, S.}, \bibinfo{author}{Wijewickrema, S.},
  \bibinfo{author}{Bailey, J.}, \bibinfo{year}{2018}.
\newblock \bibinfo{title}{Dimensionality-driven learning with noisy labels},
  in: \bibinfo{booktitle}{International Conference on Machine Learning}, pp.
  \bibinfo{pages}{3355--3364}.
\bibitem[{van~der Meer et~al.(2022)van~der Meer, Oosterlee and
  Borovykh}]{weighted_loss_10}
\bibinfo{author}{van~der Meer, R.}, \bibinfo{author}{Oosterlee, C.W.},
  \bibinfo{author}{Borovykh, A.}, \bibinfo{year}{2022}.
\newblock \bibinfo{title}{Optimally weighted loss functions for solving pdes
  with neural networks}.
\newblock \bibinfo{journal}{Journal of Computational and Applied Mathematics}
  \bibinfo{volume}{405}, \bibinfo{pages}{113887}.
\bibitem[{Menon et~al.(2015)Menon, Van~Rooyen, Ong and Williamson}]{noisy_1}
\bibinfo{author}{Menon, A.}, \bibinfo{author}{Van~Rooyen, B.},
  \bibinfo{author}{Ong, C.S.}, \bibinfo{author}{Williamson, B.},
  \bibinfo{year}{2015}.
\newblock \bibinfo{title}{Learning from corrupted binary labels via
  class-probability estimation}, in: \bibinfo{booktitle}{International
  conference on machine learning}, pp. \bibinfo{pages}{125--134}.
\bibitem[{Miyato et~al.(2018)Miyato, Maeda, Koyama and
  Ishii}]{semi-supervised_3}
\bibinfo{author}{Miyato, T.}, \bibinfo{author}{Maeda, S.i.},
  \bibinfo{author}{Koyama, M.}, \bibinfo{author}{Ishii, S.},
  \bibinfo{year}{2018}.
\newblock \bibinfo{title}{Virtual adversarial training: a regularization method
  for supervised and semi-supervised learning}.
\newblock \bibinfo{journal}{IEEE transactions on pattern analysis and machine
  intelligence} \bibinfo{volume}{41}, \bibinfo{pages}{1979--1993}.
\bibitem[{Mohri et~al.(2018)Mohri, Rostamizadeh and Talwalkar}]{rademacher}
\bibinfo{author}{Mohri, M.}, \bibinfo{author}{Rostamizadeh, A.},
  \bibinfo{author}{Talwalkar, A.}, \bibinfo{year}{2018}.
\newblock \bibinfo{title}{Foundations of machine learning}.
\bibitem[{Rezaei et~al.(2020)Rezaei, Yang and Meinel}]{application_3}
\bibinfo{author}{Rezaei, M.}, \bibinfo{author}{Yang, H.},
  \bibinfo{author}{Meinel, C.}, \bibinfo{year}{2020}.
\newblock \bibinfo{title}{Recurrent generative adversarial network for learning
  imbalanced medical image semantic segmentation}.
\newblock \bibinfo{journal}{Multimedia Tools and Applications}
  \bibinfo{volume}{79}, \bibinfo{pages}{15329--15348}.
\bibitem[{Richhariya and Tanveer(2020)}]{weighted_loss_3}
\bibinfo{author}{Richhariya, B.}, \bibinfo{author}{Tanveer, M.},
  \bibinfo{year}{2020}.
\newblock \bibinfo{title}{A reduced universum twin support vector machine for
  class imbalance learning}.
\newblock \bibinfo{journal}{Pattern Recognition} \bibinfo{volume}{102},
  \bibinfo{pages}{107150}.
\bibitem[{Sakai et~al.(2018)Sakai, Niu and Sugiyama}]{pu_3}
\bibinfo{author}{Sakai, T.}, \bibinfo{author}{Niu, G.},
  \bibinfo{author}{Sugiyama, M.}, \bibinfo{year}{2018}.
\newblock \bibinfo{title}{Semi-supervised auc optimization based on
  positive-unlabeled learning}.
\newblock \bibinfo{journal}{Machine Learning} \bibinfo{volume}{107},
  \bibinfo{pages}{767--794}.
\bibitem[{Su et~al.(2021)Su, Chen and Xu}]{pu_6}
\bibinfo{author}{Su, G.}, \bibinfo{author}{Chen, W.}, \bibinfo{author}{Xu, M.},
  \bibinfo{year}{2021}.
\newblock \bibinfo{title}{Positive-unlabeled learning from imbalanced data.},
  in: \bibinfo{booktitle}{IJCAI}, pp. \bibinfo{pages}{2995--3001}.
\bibitem[{Tang and Zhang(2017)}]{pl_1}
\bibinfo{author}{Tang, C.Z.}, \bibinfo{author}{Zhang, M.L.},
  \bibinfo{year}{2017}.
\newblock \bibinfo{title}{Confidence-rated discriminative partial label
  learning}, in: \bibinfo{booktitle}{Proceedings of the AAAI Conference on
  Artificial Intelligence}, pp. \bibinfo{pages}{2611--2617}.
\bibitem[{Tarvainen and Valpola(2017)}]{semi-supervised_2}
\bibinfo{author}{Tarvainen, A.}, \bibinfo{author}{Valpola, H.},
  \bibinfo{year}{2017}.
\newblock \bibinfo{title}{Mean teachers are better role models: Weight-averaged
  consistency targets improve semi-supervised deep learning results}.
\newblock \bibinfo{journal}{Advances in neural information processing systems}
  \bibinfo{volume}{30}, \bibinfo{pages}{1195--1204}.
\bibitem[{Torralba et~al.(2008)Torralba, Fergus and Freeman}]{CIFAR10}
\bibinfo{author}{Torralba, A.}, \bibinfo{author}{Fergus, R.},
  \bibinfo{author}{Freeman, W.T.}, \bibinfo{year}{2008}.
\newblock \bibinfo{title}{80 million tiny images: A large data set for
  nonparametric object and scene recognition}.
\newblock \bibinfo{journal}{IEEE transactions on pattern analysis and machine
  intelligence} \bibinfo{volume}{30}, \bibinfo{pages}{1958--1970}.
\bibitem[{Wang et~al.(2019)Wang, Ma, Chen, Luo, Yi and Bailey}]{noisy_6}
\bibinfo{author}{Wang, Y.}, \bibinfo{author}{Ma, X.}, \bibinfo{author}{Chen,
  Z.}, \bibinfo{author}{Luo, Y.}, \bibinfo{author}{Yi, J.},
  \bibinfo{author}{Bailey, J.}, \bibinfo{year}{2019}.
\newblock \bibinfo{title}{Symmetric cross entropy for robust learning with
  noisy labels}, in: \bibinfo{booktitle}{Proceedings of the IEEE/CVF
  International Conference on Computer Vision}, pp. \bibinfo{pages}{322--330}.
\bibitem[{Xie and Huang(2018)}]{pl_2}
\bibinfo{author}{Xie, M.K.}, \bibinfo{author}{Huang, S.J.},
  \bibinfo{year}{2018}.
\newblock \bibinfo{title}{Partial multi-label learning}, in:
  \bibinfo{booktitle}{Proceedings of the AAAI Conference on Artificial
  Intelligence}, pp. \bibinfo{pages}{4302--4309}.
\bibitem[{Xu et~al.(2020a)Xu, Gong, Chen, Liu, Zhang and Batmanghelich}]{cll_6}
\bibinfo{author}{Xu, Y.}, \bibinfo{author}{Gong, M.}, \bibinfo{author}{Chen,
  J.}, \bibinfo{author}{Liu, T.}, \bibinfo{author}{Zhang, K.},
  \bibinfo{author}{Batmanghelich, K.}, \bibinfo{year}{2020}a.
\newblock \bibinfo{title}{Generative-discriminative complementary learning},
  in: \bibinfo{booktitle}{Proceedings of the AAAI Conference on Artificial
  Intelligence}, pp. \bibinfo{pages}{6526--6533}.
\bibitem[{Xu et~al.(2020b)Xu, Gong, Chen, Liu, Zhang and
  Batmanghelich}]{application_2}
\bibinfo{author}{Xu, Y.}, \bibinfo{author}{Gong, M.}, \bibinfo{author}{Chen,
  J.}, \bibinfo{author}{Liu, T.}, \bibinfo{author}{Zhang, K.},
  \bibinfo{author}{Batmanghelich, K.}, \bibinfo{year}{2020}b.
\newblock \bibinfo{title}{Generative-discriminative complementary learning},
  in: \bibinfo{booktitle}{Proceedings of the AAAI Conference on Artificial
  Intelligence}, pp. \bibinfo{pages}{6526--6533}.
\bibitem[{Yang et~al.(2009)Yang, Yang and Wang}]{weighted_loss_6}
\bibinfo{author}{Yang, C.Y.}, \bibinfo{author}{Yang, J.S.},
  \bibinfo{author}{Wang, J.J.}, \bibinfo{year}{2009}.
\newblock \bibinfo{title}{Margin calibration in svm class-imbalanced learning}.
\newblock \bibinfo{journal}{Neurocomputing} \bibinfo{volume}{73},
  \bibinfo{pages}{397--411}.
\bibitem[{Yu et~al.(2018)Yu, Liu, Gong and Tao}]{cll_2}
\bibinfo{author}{Yu, X.}, \bibinfo{author}{Liu, T.}, \bibinfo{author}{Gong,
  M.}, \bibinfo{author}{Tao, D.}, \bibinfo{year}{2018}.
\newblock \bibinfo{title}{Learning with biased complementary labels}, in:
  \bibinfo{booktitle}{Proceedings of the European conference on computer vision
  (ECCV)}, pp. \bibinfo{pages}{68--83}.
\bibitem[{Zhang et~al.(2019)Zhang, Du, Yoshida and Yang}]{weighted_loss_8}
\bibinfo{author}{Zhang, W.}, \bibinfo{author}{Du, Y.},
  \bibinfo{author}{Yoshida, T.}, \bibinfo{author}{Yang, Y.},
  \bibinfo{year}{2019}.
\newblock \bibinfo{title}{Deeprec: A deep neural network approach to
  recommendation with item embedding and weighted loss function}.
\newblock \bibinfo{journal}{Information Sciences} \bibinfo{volume}{470},
  \bibinfo{pages}{121--140}.
\bibitem[{Zhang et~al.(2021)Zhang, Zhang, Cao and Zhang}]{pl_5}
\bibinfo{author}{Zhang, Z.R.}, \bibinfo{author}{Zhang, Q.W.},
  \bibinfo{author}{Cao, Y.}, \bibinfo{author}{Zhang, M.L.},
  \bibinfo{year}{2021}.
\newblock \bibinfo{title}{Exploiting unlabeled data via partial label
  assignment for multi-class semi-supervised learning}, in:
  \bibinfo{booktitle}{Proceedings of the AAAI Conference on Artificial
  Intelligence}, pp. \bibinfo{pages}{10973--10980}.

\end{thebibliography}






\appendix
\section{The Proof of Lemma 1}
\begin{lemma}
	For any $ \delta > 0 $, with the probability at least $ 1- \delta / 2 $, we have 
	\begin{equation}
		\begin{split}
			\mathop{sup}_{f\in\mathcal{F}} | \hat{R}(f) - R(f) | \leqslant 2 \omega_{K} L_{\phi}  [ \pi_{K} (K-1) + K] \mathfrak{R}_{N_l}(\mathcal{F}) 
			+ \sqrt{\frac{I^2log\frac{4}{\delta}}{2N_l}}.
		\end{split}
		\nonumber
	\end{equation}
\end{lemma}

\begin{proof}\let\qed\relax
	Let According to McDiarmid inequality, for $ \forall \varepsilon > 0 $ we have 
	\begin{equation}
		P(|R(f) - \hat{R}(f) - \mathbb{E}_{\bar{\mathcal{D}}_L}[R(f) - \hat{R}(f)] | \geqslant \varepsilon) \leqslant 2 exp(\frac{-2\varepsilon^2}{\sum_{N_l}C^2_i})
		\nonumber
	\end{equation}
	where $ \bar{\mathcal{D}}_L = \{ (x_i, \bar{y}_i) \}_{i=1}^{N_l} $. If $ x_i $ in $ \bar{\mathcal{D}}_L $ is replaced with $ x'_i $, we have $ \bar{\mathcal{D}}'_L $ and
	\begin{equation}
		\begin{aligned}
			C_i &= \mathop{sup}_{f\in\mathcal{F}} |[\hat{R}(f)-R(f)] -  [\hat{R}'(f) - R(f)] | \\
			& = \mathop{sup}_{f\in\mathcal{F}} |\hat{R}(f)-\hat{R}'(f) | \\
			& = \mathop{sup}_{f\in\mathcal{F}}  \frac{1}{N_l} \{ [\pi_z(K-1)\omega_z[f(x_i, z)]  + \sum_{j=1}^{K} \omega_j[f(x_i, j)]] - [\pi_z(K-1)\omega_z[f(x'_i, z)]  + \sum_{j=1}^{K} \omega_j[f(x'_i, j)]]  \} \\ 
			& \leqslant  \frac{1}{N_l}\{2\omega_KL_f[\pi_K(K-1) + K]\} = \frac{I}{N_l}
		\end{aligned}
		\nonumber
	\end{equation}
	where $ \pi_z \leqslant \pi_K, \omega_z \leqslant \omega_K$, $ L[f(x_i, z)] \leqslant L_f $, $ I = 2\omega_KL_f[\pi_K(K-1) + K] $, and $ \hat{R}'(f) = \mathbb{E}_{\bar{\mathcal{D}}'_L} $.
	
	Let $ 1 - P(|R(f) - \hat{R}(f) - \mathbb{E}_{D_l}[R(f) - \hat{R}(f)] | \leqslant \varepsilon) \leqslant \frac{\delta}{2} $, we have
	\begin{equation}
		2 exp(\frac{-2\varepsilon^2}{\sum_{N_l}C^2_i}) =  \frac{\delta}{2} \Rightarrow \varepsilon = \sqrt{\frac{I^2log\frac{4}{\delta}}{2N_l}}
		\nonumber
	\end{equation}
	Then, we can obtain the
	\begin{equation}
		| \hat{R}(f) - R(f) - \mathbb{E}_{\bar{\mathcal{D}}_L}[\hat{R}(f) - R(f)]| \leqslant \sqrt{\frac{I^2log\frac{4}{\delta}}{2N_l}}
		\nonumber
	\end{equation}
	with the probability at least $ 1 - \frac{\delta}{2} $. 
	Therefore, 
	\begin{equation}
		\begin{aligned}
			\mathop{sup}_{f\in\mathcal{F}}| \hat{R}(f) - R(f) | & \leqslant \mathbb{E}_{\bar{\mathcal{D}}_L}[\hat{R}(f) - R(f)] + \sqrt{\frac{I^2log\frac{4}{\delta}}{2N_l}} \\
			& \leqslant 2\mathfrak{R}_{N_l} (\hat{\ell} o \mathcal{F}) + \sqrt{\frac{I^2log\frac{4}{\delta}}{2N_l}} \\ 
			& \leqslant 2\omega_KL_\phi[\pi_K(K-1) + K] \mathfrak{R}_{N_l}(\mathcal{F}) 
			+ \sqrt{\frac{I^2log\frac{4}{\delta}}{2N_l}},
		\end{aligned}
		\nonumber
	\end{equation}
	which proves the Lemma 1.
\end{proof}

\section{The Proof of Theorem 1}
\begin{theorem}
	For any $ \delta > 0 $, with the probability at least $ 1- \delta / 2 $, we have 
	\begin{equation}
		\begin{split}
			R(\hat{f}) - R(f^{\ast})  \leqslant 4 \omega_{K} L_{\phi}  [ \pi_{K} (K-1) + K] \mathfrak{R}_{N_l}(\mathcal{F}) 
			+ \sqrt{\frac{2I^2log\frac{4}{\delta}}{N_l}}.
		\end{split}
		\nonumber
	\end{equation}
\end{theorem}

\begin{proof}\let\qed\relax
	\begin{equation}
		\begin{aligned}
			R(\hat{f}) - R(f^*) & = (\hat{R}(\hat{f}) - \hat{R}(f^*)) + (R(\hat{f})-\hat{R}(\hat{f})) + (\hat{R}(f^*) - R(f^*))
		\end{aligned}
		\nonumber
	\end{equation}
	with $ N_l \rightarrow \infty $, we have $ \hat{R}(\hat{f}) - \hat{R}(f^*) = 0 $. By using $ R(\hat{f}) - \hat{R}(\hat{f}) \leqslant | R(\hat{f}) - \hat{R}(\hat{f}) |, \hat{R}(f^*) - R(f^*) \leqslant |\hat{R}(f^*) - R(f^*)| $, and Lemma 1, we have 
	\begin{equation}
		\begin{aligned}
			R(\hat{f}) - R(f^*) & \leqslant 0 + 2 \mathop{sup}_{f\in\mathcal{F}} | \hat{R}(f) - R(f) | \\
			& \leqslant 4\omega_KL_\phi [ \pi_{K} (K-1) + K] \mathfrak{R}_{N_l}(\mathcal{F}) 
			+ \sqrt{\frac{2I^2log\frac{4}{\delta}}{N_l}}
		\end{aligned}
		\nonumber
	\end{equation}
	with the probability at least $ 1 - \frac{\delta}{2} $, which finishes the proof.
\end{proof}

\end{document}